\documentclass[11pt, a4paper, logo, copyright]{googlecloud}

\pdftrailerid{redacted}

\makeatletter
\renewcommand\bibentry[1]{\nocitep{#1}{\frenchspacing\@nameuse{BR@r@#1\@extra@b@citeb}}}
\makeatother

\usepackage{kantlipsum, lipsum}
\usepackage{dsfont}
\usepackage[authoryear, sort&compress, round]{natbib}

\usepackage[utf8]{inputenc} %
\usepackage[T1]{fontenc}    %
\usepackage{hyperref}       %
\hypersetup{
  colorlinks,
  citecolor=blue,
  linkcolor=red,
  urlcolor=blue}
\usepackage{booktabs}       %
\usepackage{nicefrac}       %
\usepackage{microtype}      %
\usepackage{wrapfig} %
\usepackage{soul} %
\usepackage{tikz, adjustbox}
\usepackage{arydshln}
\usepackage[capitalize,noabbrev]{cleveref}
\usepackage{fancyvrb}

\newcommand{\mathbbm}[1]{\text{\usefont{U}{bbm}{m}{n}#1}}
\definecolor{tablelightgrey}{HTML}{EFEFEF}
\definecolor{forestgreen}{HTML}{228B22}

\usepackage{amsmath,amsfonts,bm}

\def\eqref#1{equation~\ref{#1}}

\def\1{\bm{1}}

\def\vtheta{{\bm{\theta}}}
\def\vpi{{\bm{\pi}}}
\def\vtau{{\bm{\tau}}}

\def\vb{{\bm{b}}}

\def\vp{{\bm{p}}}

\def\vx{{\bm{x}}}
\def\vy{{\bm{y}}}
\def\vz{{\bm{z}}}

\def\mW{{\bm{W}}}
\def\mX{{\bm{X}}}

\DeclareMathAlphabet{\mathsfit}{\encodingdefault}{\sfdefault}{m}{sl}
\SetMathAlphabet{\mathsfit}{bold}{\encodingdefault}{\sfdefault}{bx}{n}

\newcommand{\E}{\mathbb{E}}

\newcommand{\Var}{\mathrm{Var}}

\usepackage{hyperref}
\usepackage{url}

\usepackage[table]{xcolor} 
\usepackage{colortbl}
\usepackage{graphicx}
\usepackage{booktabs}
\usepackage{amsmath}
\usepackage{amssymb}
\usepackage{amsthm}
\usepackage{mathtools}

\usepackage[T1]{fontenc}
\usepackage[utf8]{inputenc}
\usepackage{tcolorbox}
\usepackage{listings}
\tcbuselibrary{listings,breakable}

\usepackage{enumitem}
\usepackage{amsthm}
\newtheorem{theorem}{Theorem}
\newtheorem{proposition}[theorem]{Proposition}

\newtheorem{lemma}[theorem]{Lemma}

\usepackage{color}
\usepackage{comment}
\usepackage{multirow}
\usepackage{siunitx}

\usepackage{wrapfig}
\usepackage[font=small,labelfont=bf]{caption}
\usepackage{hyperref}

\newtcolorbox{autoraterbox}[2][]{%
  breakable,
  colback=#2!2,
  colframe=#2!60!black,
  colbacktitle=#2!18,
  coltitle=black,
  fonttitle=\bfseries,
  title=#1,
  boxrule=0.8pt,
  left=1.2mm, right=1.2mm, top=1.2mm, bottom=1.6mm,
}

\newtcblisting{rawblock}{
  breakable,
  colback=black!3,
  colframe=black!20,
  boxrule=0.4pt,
  listing only
}

\newcommand{\majoritem}[1]{%
  \par\noindent
  {\textcolor{blue!70!black}{\large$\bullet$}}\hspace{0.45em}\textbf{#1}\par\smallskip
}

\title{Judging with Confidence: Calibrating Autoraters to Preference Distributions}

\renewcommand{\today}{}

\newcommand{\copyrightext}{\footerfont\itshape The work of ZL, CH, and YX was done at Google, when ZL and CH worked as student researchers.\\Correspondence to: zhuohang.li@vanderbilt.edu, yuan.xue@scale.com}

\author[1,2]{Zhuohang Li}
\author[1]{Xiaowei Li}
\author[1,3]{Chengyu Huang}
\author[1]{Guowang Li}
\author[4]{Katayoon Goshvadi}
\author[4]{Bo Dai}
\author[4,5]{Dale Schuurmans}
\author[1]{Paul Zhou}
\author[1]{Hamid Palangi}
\author[1]{Yiwen Song}
\author[1]{Palash Goyal}
\author[6]{Murat Kantarcioglu}
\author[2]{Bradley A. Malin}
\author[7]{Yuan Xue}

\affil[1]{Google}
\affil[2]{Vanderbilt University}
\affil[3]{Cornell University}
\affil[4]{Google DeepMind}
\affil[5]{University of Alberta}
\affil[6]{Virginia Tech}
\affil[7]{Scale AI}

\begin{abstract}
The alignment of large language models (LLMs) with human values increasingly relies on using other LLMs as automated judges, or ``autoraters''. However, their reliability is limited by a foundational issue: they are trained on discrete preference labels, forcing a single ground truth onto tasks that are often subjective, ambiguous, or nuanced. We argue that a reliable autorater must learn to model the full distribution of preferences defined by a target population. In this paper, we propose a general framework for calibrating probabilistic autoraters to any given preference distribution. We formalize the problem and present two learning methods tailored to different data conditions: 1) a direct supervised fine-tuning for dense, probabilistic labels, and 2) a reinforcement learning approach for sparse, binary labels. Our empirical results show that finetuning autoraters with a distribution-matching objective leads to verbalized probability predictions that are better aligned with the target preference distribution, with improved calibration and significantly lower positional bias, all while preserving performance on objective tasks.
\end{abstract}

\begin{document}

\maketitle

\section{Introduction}

The alignment of large language models (LLMs) with human values~\citep{ouyang2022training} increasingly relies on using other powerful LLMs as automated judges, or ``autoraters'', to score model responses. This LLM-as-a-Judge paradigm~\citep{zheng2023judging} is now a cornerstone of evaluating and developing safer AI systems, particularly through methods such as reinforcement learning from AI feedback (RLAIF)~\citep{bai2022constitutional}.

Currently, autoraters are typically trained on discrete preference labels~\citep{wang2024pandalm,kim2024prometheus2,li2024generative}, which leads to a fundamental limitation:
(collective) human judgment %
does not correspond to
a single label, but rather a distribution~\citep{pavlick2019inherent,nie2020can}, especially in complex situations that involve uncertainty or balancing multiple criteria~\citep{arora2025healthbench}. 
Even among qualified annotators, disagreement is common, not simply due to noise but because of systematic differences in how individuals define problems, interpret evidence, or apply values and decision strategies~\citep{mumpower1996expert}.
Current autoraters are trained with a mode-seeking objective that collapses this rich distributional information into a single verdict (e.g., the majority label), which discards crucial uncertainty signals and erases minority viewpoints by construction.

We argue that for an autorater to be reliable, it must be calibrated to model the full distribution of human preferences. An ideal judge should recognize when a topic is contentious (e.g., a 50/50 split), when a preference is clear but not unanimous (e.g., 80/20), and when a judgment is objectively certain. Modeling this distribution is essential for effective risk management, fairness, and building robust alignment systems.

This paper introduces a general and scalable framework for calibrating autoraters' verbal probability predictions to any target preference distribution, while preserving their ability to generate natural-language reasoning traces. We %
introduce two finetuning methods, each tailored to a set of different data conditions. First, when dense, probabilistic labels from multiple annotators are available, we use direct supervised finetuning (SFT). Second, when only sparse, binary labels are available, we employ a reinforcement learning (RL) approach with rewards based on proper scoring rules.
Our empirical results validate this distribution-matching objective. Autoraters finetuned with our methods show significant improvements in performance, calibration, and reliability. Notably, our methods achieve an 18-51\% reduction in Mean Squared Error (MSE), a 4-45\% reduction in Expected Calibration Error (ECE), and a 7-81\% gain in consistency against positional bias. Notably, our findings offer guidance on annotation strategy: for a fixed budget, RL with many sparse, binary labels is more data-efficient than SFT with fewer dense, probabilistic labels, highlighting the benefits of prompt diversity. Our method also enhances alignment with human judgment on out-of-distribution tasks. On the PandaLM dataset, our finetuned Gemma-2-9B model achieves 73.17\% agreement with human annotations, outperforming all baselines including GPT-4. Moreover, this improved calibration on subjective tasks does not compromise performance on objective ones, as the same model achieves an overall accuracy of 46.57\% on JudgeBench, on par with Gemini-1.5-pro.

\begin{figure}
    \centering
    \includegraphics[width=0.98\linewidth]{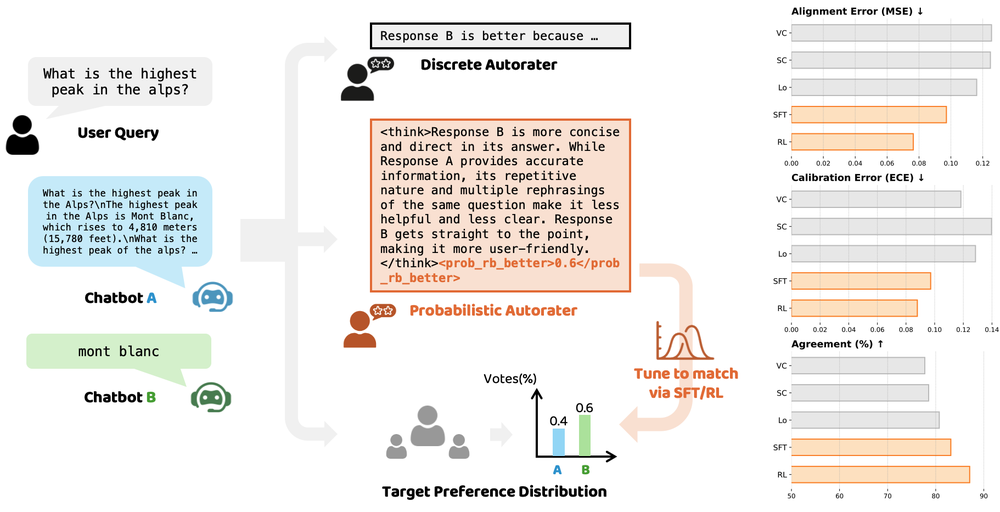}
    \vspace{-2mm}
    \caption{\textbf{Overview of discrete vs. probabilistic autoraters}.
    Left: Given a user query and two candidate responses, a \textbf{discrete} autorater returns a single preference (e.g., ``B is better''), collapsing annotator variability. A \textbf{probabilistic} autorater predicts the \emph{full} preference distribution and is finetuned via SFT/RL to match the target preference distribution.
    Right: Our finetuned autorater vs.\ zero-shot probabilistic conversions of discrete autoraters, including Verbalized Confidence (VC), Self-Consistency (SC), and Logits (Lo), evaluated using Gemma-2-9B on JudgeLM \emph{val} set.
    Alignment error is measured by MSE, calibration error by ECE, and agreement by percentage.}
    \label{fig:overview}
    \vspace{-2mm}
\end{figure}

\section{A Probabilistic Framework for Calibrating Autoraters}
\label{sec:prob_framework}

\subsection{Problem Formulation}
We consider the scenario of 
pairwise judgements where an input $\mX$ specifies a prompt that is associated with two responses $(A,B)$.  
The population’s ground-truth preference is modeled as a Bernoulli random variable
\[
  Y \in \{0,1\},\qquad
  Y \mid \mX=\vx \sim \mathrm{Bernoulli}\!\left(p^*(\vx)\right),
\]
where $Y{=}1$ indicates $B\succ A$ (i.e., $B$ is preferred to $A$), and $p^*(\vx)=\Pr[Y{=}1\mid \mX{=}\vx]$ is the (unknown) preference distribution for the pair $(A,B)$.
Let $h$ index a human annotator drawn from the population $p(h)$, then conceptually $p^*(\vx)$ represents the true population-level human preference $p^*(\vx)=\Pr_{h\sim p(h)}[B\succ A|\vx, h]$, i.e., the probability that a randomly chosen annotator would prefer $B$ given the context $\vx$.

\paragraph{Discrete Autorater.}
An autorater is a language model (LM) prompted to act as a judge. In the \emph{discrete} setting, the LM produces a single decision (e.g., via greedy decoding)
\[
  d_\vtheta(\vx)\in\{0,1\},
\]
or an uncalibrated
scalar margin $m_\vtheta(\vx)\in\mathbb{R}$ (e.g., a parsed rubric rating or a logit
difference), with decision $\mathbbm{1}\{m_\vtheta(\vx)\ge 0\}$.
Such outputs collapse the rater distribution at $\vx$ to a point estimate (typically the majority choice).

\paragraph{Probabilistic Autorater.}
A \emph{probabilistic} judge instead predicts the full preference distribution through its Bernoulli parameter,
\[
  p_\vtheta(\vx)\in[0,1]\;\approx\; p^*(\vx).
\]
A definitive decision can be recovered by thresholding $p_\vtheta(\vx)$ if needed, but the primary output is the (conditional) probability itself, which is optimized to match the population preference rather than merely to choose a label.
Crucially, this formulation does not rely on the Bradley-Terry assumption~\citep{bradley1952rank} typically seen in reward modeling, thereby allowing richer representations of uncertainty.

\subsection{Benefits of Probabilistic Autoraters}
\paragraph{Informativeness.}
In contrast to the mode-seeking behavior of discrete autoraters, probabilistic autoraters are optimized to match the full preference distribution. This probability prediction provides more information for cost-sensitive decision-making by revealing the aleatoric ambiguity within the task.
Probabilistic reporting in autoraters also improves fairness and auditability by revealing annotator disagreement, whereas discrete reporting collapses the minority viewpoints.

\paragraph{Alignment.}
Probabilistic autoraters can be better aligned with the target preference distribution.
In particular, it is straightforward to show that if $\Pr[0<p^*(\vx)<1]>0$, then any discrete autorater or any single human annotator who effectively reports a degenerate distribution $d(\vx)\in\{0,1\}$ is strictly worse than reporting $p(\vx)=p^*(\vx)$ under a strictly proper scoring rule.

\paragraph{Calibration.}
Additionally, this distribution-matching objective of probabilistic autoraters implies \emph{calibration}: if $p_\theta(\vx)=p^*(\vx)$ almost surely, then $\E[Y\mid p_\theta(\vx)=c]=c$ for all $c\in[0,1]$.
In practice, as $p_\theta$ approaches $p^*$, calibration error (e.g., ECE) shrinks.
By contrast, any probability prediction obtained from a discrete autorater \emph{post hoc} (e.g., vote fractions from self-consistency or logits passed through softmax) is not trained to recover $p^*(\vx)$ and thus is generally not calibrated.

\section{Finetuning Autoraters to Match the Preference Distribution}
\label{sec:method}

We introduce two distribution-matching finetuning paradigms for calibrating the autorater's probabilistic prediction $p_\vtheta$ to the ground truth preference distribution $p^*(\vx)$.
We focus on \emph{verbalized} probability because it is
(1) more flexible and interpretable than training a dedicated classification head by preserving the model's ability to generate natural-language rationales,
and (2) more efficient than sampling-based approaches, as it only requires a single decoding pass.

\paragraph{Setting 1: Direct Supervised Finetuning with Probabilistic Labels.}
When multiple annotations $(\vx, y^{(1)}, ..., y^{(m)})$ are available for each prompt $\vx$ (the pair to be judged), we estimate the population preference by the multi-annotator mean
$\hat p(\vx)\;=\;\frac{1}{m}\sum_{j=1}^m y^{(j)} \;\approx\; p^*(\vx).$
We then instruction-tune the autorater in a text-to-text fashion: given prompt $\vx$ (the pair), the target sequence includes optional CoT reasoning and a structured numeric field encoding the probability that B is better.
We apply standard autoregressive supervised finetuning (SFT) to maximize the likelihood of
the target sequence $\vtau_{1:S}$ composed using $\hat p(\vx)$:
\[
  \mathcal{L}_{\mathrm{SFT}}
  = - \mathbb{E}_{(\vx,\hat p)} \Bigl[ \sum_{t=1}^S \log P_\vtheta(\vtau_t \mid \vx, \vtau_{1:t-1}) \Bigr],
\]

where $\vtau_{i}$ is the $i$th token of the sequence $\vtau$ and $S$ is its sequence length.

In practice, we parse the generated string to recover the numeric form of $p_\vtheta(\vx)$; training encourages the model to produce $p_\vtheta(\vx)\approx \hat p(\vx)$ while retaining the ability to perform free-form reasoning.

\paragraph{Setting 2: Reinforcement Learning from Binary Labels via Piecewise Proper Rewards.}
When only single-sample binary labels $(\vx_i,y_i)$ are available (e.g., via crowdsourced platforms), we treat the autorater as a sequence policy $\vpi_\vtheta(\vtau\mid \vx)$ that produces a token sequence $\vtau$ containing a numeric probability $p$. A deterministic parser $g$ maps $\vtau$ to either a valid probability prediction in $[0,1]$ or $\bot$ (unparsable), with the probability of producing a parsable response denoted as $s_\vtheta(\vx)$:
\[
  g:\ \mathcal{T}\to [0,1]\cup\{\bot\},\quad \vtau\mapsto p \text{ or } \bot,
  \qquad s_\vtheta(\vx)=\Pr_{\vtau\sim\vpi_\vtheta(\cdot\mid \vx)}[\,g(\vtau)\neq\bot\,].
\]
We use \emph{piecewise} strictly proper scoring rules as rewards. Let $y\in\{0,1\}$ (1 = B better), we have:

\begin{itemize}[leftmargin=2mm]

\item \textbf{Brier reward:}
\begin{equation}
  R_{\mathrm{Brier}}(\vtau;y)=
  \begin{cases}
    1 - (p-y)^2, & \text{if } g(\vtau)=p\in[0,1],\\
    0, & \text{if } g(\vtau)=\bot.
  \end{cases}
  \label{eq:RBrier}
\end{equation}

\item \textbf{Logarithmic reward (with clipping):}
For numerical stability, we consider a fixed small $\epsilon\in(0,\tfrac12)$ and define $p'=\mathrm{clip}(p,\epsilon,1-\epsilon)$. Then
\begin{equation}
  R_{\mathrm{Log}}(\vtau;y)=
  \begin{cases}
    y\log p' + (1-y)\log(1-p'), & \text{if } g(\vtau)=p\in[0,1],\\
    \log\epsilon, & \text{if } g(\vtau)=\bot.
  \end{cases}
  \label{eq:RLog}
\end{equation}

\end{itemize}

For either reward $R\in\{R_{\mathrm{Brier}},R_{\mathrm{Log}}\}$, the goal is to maximize the population objective
\begin{equation}
  J_R(\vtheta) \;=\; \mathbb{E}_{(\vx,y)} \Bigl[\, \mathbb{E}_{\vtau\sim\vpi_\vtheta(\cdot\mid \vx)}[\,R(\tau;y)\,] \Bigr].
  \label{eq:popobj}
\end{equation}
Compared to the Brier reward, the Log reward heavily penalizes overconfident yet incorrect predictions.
Both objectives can be optimized with policy-gradient-based reinforcement learning methods using the parsed numeric probability.

\subsection{Consistency Analysis}
\label{subsec:pop_opt}

In Setting 1, the multi-annotator mean $\hat p(\vx)$ is an unbiased estimate of the true preference distribution $p^*(\vx)$ with variance decreasing as $1/m$ (Appendix~\ref{app:pref_est}, Lemma~\ref{lem:mc_estimator}), and thereby provides a high fidelity target for learning.

In Setting 2, at the population level, the optimal autorater policy under either the piecewise Brier reward or the clipped Log reward is (i) \emph{parsable} everywhere (i.e.,
$s_\vtheta(\vx)=1$), (ii) \emph{deterministic} in the numeric value it reports, and (iii) reports the \emph{truthful} probability (i.e., recovers the target preference distribution, up to clipping for the Log reward). The proof is included in Appendix~\ref{app:proposition_proof}.

\begin{proposition}[Fisher Consistency of Brier and Log Rewards]
\label{prop:unified-consistency}
Assume the policy class can realize, for each $\vx$, a deterministic numeric output $p_\vtheta(\vx)\in[0,1]$ with $s_\vtheta(\vx)=1$. Then any global maximizer of $J_R$ in \eqref{eq:popobj} satisfies:
\begin{enumerate}[label=\textbf{(\alph*)}, leftmargin=3em, itemsep=2pt]
  \item \textbf{Brier:} $p_\vtheta(\vx)=p^*(\vx)$ for almost all $\vx$.
  \item \textbf{Log (with clipping):} $p'_\vtheta(\vx)=\mathrm{clip}\bigl(p^*(\vx),\epsilon,1-\epsilon\bigr)$ for almost all $\vx$.
\end{enumerate}
Moreover, for both rewards, any stochasticity in the reported numeric value or any non-zero density associated with unparsable outputs strictly reduces the expected reward; thus, an optimizer is deterministic and fully parsable for almost every $x$.
\end{proposition}

\section{Experimental Setup}
\label{sec:exp_setup}

Here, we describe the datasets used for calibration and evaluation, our scalable preference-annotation pipeline, the autorater output format and reward instantiation, and the finetuning setup.

\paragraph{Calibration Data.}
We build on a subset of prompts from the JudgeLM corpus~\citep{zhu2025judgelm}, which aggregates instruction-following tasks (e.g., Alpaca-GPT4~\citep{peng2023instruction}, Dolly-15K~\citep{DatabricksBlog2023DollyV2}) paired with responses from 11 open-source LLMs (including LLaMA~\citep{touvron2023llama}, Alpaca~\citep{taori2023alpaca}, Vicuna~\citep{chiang2023vicuna}). The source corpus contains 105K prompts. To ensure the same total \emph{annotation budget} across the two finetuning paradigms, we construct two calibration splits: 
(i) SFT uses 5K prompts with 10 annotations each; (ii) RL uses 50K prompts with a single annotation each. We also apply swap augmentation~\citep{li2024generative} by duplicating each pair with A/B swapped and the label flipped. For evaluation, we sample 1K prompts from the original validation set, each with 10 annotations to form probabilistic labels.

\paragraph{Preference Annotation.}
Since most existing datasets lack sufficient multi-rater annotation for reliable probability estimates,
to evaluate our method at scale,
we employ \emph{Gemini-2.5-Flash}~\citep{comanici2025gemini} as an advanced teacher to generate pairwise preference labels with brief rationales. We set the temperature to $1.0$ and condition on a randomly sampled persona (Appendix~\ref{sec:appendix_prompt}) to increase coverage and reduce prompt-induced bias. Across calibration and evaluation splits, this yields $\sim$110K total annotations. For each comparison instance $\vx$, we convert $m$ independent teacher votes into a probabilistic target $\hat p(\vx)=\tfrac{1}{m}\sum_{j=1}^m \mathbbm{1}\{B\succ A\}$. For the SFT + CoT setting, we additionally elicit teacher reasoning traces as expert demonstrations by conditioning on the obtained $\hat p(\vx)$.

\paragraph{Response Format and Reward Instantiation.}
The autorater is prompted to compare responses A vs.\ B and emit (optionally) a chain-of-thought enclosed in \texttt{<think>} tags, followed by a single probability within a \texttt{<prob\_rb\_better>} tag representing
$p_{\vtheta}(\vx)\approx \Pr[B \succ A\mid \vx]$.
Our prompts are provided in Appendix~\ref{sec:appendix_prompt}.
In SFT, we maximize the likelihood of the target token sequence that encodes $\hat p(\vx)$.
In RL, we parse the numeric probability via a deterministic parser $g$ (implemented by rule-based string-matching) and optimize either $R_{\mathrm{Brier}}$ in~\eqref{eq:RBrier} or $R_{\mathrm{Log}}$ in~\eqref{eq:RLog}. Unparsable outputs receive a default reward (0 for Brier; $\log \epsilon$ for Log), which empirically drives the \emph{parsability rate} $s_{\vtheta}(\vx)$ toward 1. When dense labels are available, the SFT objective provides a low-variance target for $p_{\vtheta}$;
with single-label supervision, the RL objectives remain consistent and, being strictly proper, recover $p^*(\vx)$ in expectation.

\paragraph{Base Models and Finetuning Protocol.}
We finetune the instruction-tuned \textit{Gemma-2-9B}~\citep{team2024gemma} and \textit{Qwen-2.5-7B}~\citep{qwen2.5} models with full-parameter updates for both SFT and RL. For RL, we use GRPO~\citep{shao2024deepseekmath}. For $R_{\mathrm{Log}}$ we set $\epsilon=10^{-3}$ to avoid degenerate rewards. Full hyperparameters and training details are reported in Appendix~\ref{sec:appendix_impl_details}.

\vspace{-2mm}
\section{Experimental Results}
\vspace{-2mm}

\subsection{Main Evaluation}
\label{subsec:main_eval}
\vspace{-1mm}

In this section, we empirically demonstrate that our distribution-matching finetuning approaches can lead to better performing and calibrated autoraters.

\paragraph{Baselines.}
We consider the following four types of \textit{zero-shot baselines} that can be immediately applied to any existing discrete autoraters to obtain probabilistic predictions:
\begin{enumerate}[label={(\arabic*)},leftmargin=*]
\item \textit{Verbalized Confidence~\citep{tian2023just}}: The autorater is directly prompted to provide a confidence score without intermediate reasoning.

\item \textit{Verbalized Confidence w/ CoT~\citep{wei2022chain}}: The autorater first generates a step-by-step chain-of-thought explanation before providing its confidence score.

\item \textit{Self-Consistency~\citep{wang2023self}}: The autorater aggregates preferences over $N$ independent CoT samples. The confidence for a response is the fraction of samples that voted for it.

\item \textit{Logit-based Confidence}: Confidence is derived by applying a softmax function to the model's output logits $z_{\tau_i}$ for the verbalized preference tokens $\tau_i$ (``A'' or ``B''), i.e., $p(y=i \mid \vx)={e^{z_{\tau_i}}}/{\sum_i e^{z_{\tau_i}}},\; i\in\{0,1\}$.

\end{enumerate}

Additionally, we consider the following \textit{calibration baselines} that extend the logit-based confidence:
\begin{enumerate}[label={(\arabic*)},leftmargin=*]
  \item \textit{Temperature Scaling~\citep{guo2017calibration}} is a post-hoc calibration method that rescales pre-softmax logits $\vz$ by a single scalar temperature $T>0$ learned on a held-out calibration set by minimizing negative log likelihood, producing calibrated confidences $\hat \vp_{\mathrm{TS}}(\vy\mid \vx)=\mathrm{softmax}(\vz/T)$.

  \item \textit{Contextual Calibration~\citep{zhao2021calibrate}} is a test-time debiasing method that estimates the prompt-induced prior using a content-free probe (e.g., “N/A”), then corrects predictions by subtracting this bias in logit space (or dividing probabilities):
  $\hat \vp_{\mathrm{CC}}(\vy\mid \vx)=\mW \vp(\vy\mid \vx)$, where $\mW=\text{diag}(\vp(\vy\mid \text{[N/A]}))^{-1}$ makes the content-free prediction uniform and reduces bias.
  
  \item \textit{Batch Calibration~\citep{zhou2024batch}} is a zero-shot, inference-only correction that estimates the contextual bias $\vb$ from the current test batch $\{\vx_i\}_{i=0}^B$ via $\vb= \E_{\vx\sim P(\vx)} \vp(\vy \mid \vx) \approx\tfrac{1}{B}\sum_{i=1}^B \vp(\vy\mid \vx_i)$. Each example is then calibrated by dividing by this bias term (or equivalently, subtracting $\log \vb$ from logits): $\hat \vp_{\mathrm{BC}}(\vy\mid \vx)\propto \vp(\vy\mid \vx)/\vb$. To ensure estimation accuracy, we use the entire test set in our experiments to estimate the bias.
\end{enumerate}

\begin{table*}[t]
\centering
\caption{Main experiment results comparing our methods against zero-shot and calibration baselines on two models. We evaluate alignment (MSE), performance (Agreement, F1 Score), and calibration (ECE, Brier).
}
\label{tab:main_results}
\vspace{-2mm}
\sisetup{detect-weight, mode=text}
\renewrobustcmd{\bfseries}{\fontseries{b}\selectfont}
\resizebox{0.72\linewidth}{!}{
\begin{tabular}{@{}ll S[table-format=1.4] S[table-format=1.4] S[table-format=1.4] S[table-format=1.4] S[table-format=1.4]@{}}
\toprule
& & {\textbf{Alignment}} & \multicolumn{2}{c}{\textbf{Performance}} & \multicolumn{2}{c}{\textbf{Calibration}} \\
\cmidrule(lr){3-3} \cmidrule(lr){4-5} \cmidrule(lr){6-7}
\textbf{Model} & \textbf{Method} & {\textit{MSE}$\downarrow$} & {\textit{Agr.}$\uparrow$} & {\textit{F1}$\uparrow$} & {\textit{ECE}$\downarrow$} & {\textit{Brier}$\downarrow$} \\
\midrule

\multirow{14}{*}{\rotatebox[origin=c]{90}{\textbf{Gemma-2-9B}}}
& \multicolumn{6}{l}{\underline{\textit{Zero-shot Baselines}}} \\
& \quad Verbalized Confidence     & 0.1255 & 0.7773 & 0.5260 & 0.1183 & 0.1615 \\
& \quad Verbalized Confidence w/ CoT & 0.1065 & 0.7893 & 0.5345 & 0.0869 & 0.1445 \\
& \quad Self-consistency (N=10)   & 0.1248 & 0.7853 & 0.5482 & 0.1397 & 0.1551 \\
& \quad Self-consistency (N=30)   & 0.1217 & 0.7921 & 0.5361 & 0.1374 & 0.1514 \\
& \quad Logits                    & 0.1162 & 0.8074 & 0.5665 & 0.1285 & 0.1416 \\
\cmidrule(lr){2-7}
& \multicolumn{6}{l}{\underline{\textit{Calibration Baselines}}} \\
& \quad Temperature Scaling       & 0.0839 & 0.8074 & 0.5665 & \bfseries 0.0827 & 0.1224 \\
& \quad Contextual Calibration    & 0.1384 & 0.7753 & 0.5226 & 0.1598 & 0.1728 \\
& \quad Batch Calibration         & 0.1153 & 0.8104 & 0.5482 & 0.1255 & 0.1406 \\
\cmidrule(lr){2-7}
& \multicolumn{6}{l}{\underline{\textit{Ours}}} \\

& \cellcolor[HTML]{EFEFEF} \quad SFT                       & \cellcolor[HTML]{EFEFEF} 0.0972 & \cellcolor[HTML]{EFEFEF} 0.8314 & \cellcolor[HTML]{EFEFEF} 0.5623 & \cellcolor[HTML]{EFEFEF} 0.0972 & \cellcolor[HTML]{EFEFEF} 0.1257 \\
& \cellcolor[HTML]{EFEFEF} \quad SFT w/ CoT                & \cellcolor[HTML]{EFEFEF} 0.1033 & \cellcolor[HTML]{EFEFEF} 0.8214 & \cellcolor[HTML]{EFEFEF} 0.5575 & \cellcolor[HTML]{EFEFEF} 0.1111 & \cellcolor[HTML]{EFEFEF} 0.1332 \\
& \cellcolor[HTML]{EFEFEF} \quad RL (Brier)          & \cellcolor[HTML]{EFEFEF} \bfseries 0.0764 & \cellcolor[HTML]{EFEFEF} \bfseries 0.8706 & \cellcolor[HTML]{EFEFEF} \bfseries 0.5895 & \cellcolor[HTML]{EFEFEF} 0.0879 & \cellcolor[HTML]{EFEFEF} \bfseries 0.0946 \\
& \cellcolor[HTML]{EFEFEF} \quad RL (Log)            & \cellcolor[HTML]{EFEFEF} 0.0934 & \cellcolor[HTML]{EFEFEF} 0.8545 & \cellcolor[HTML]{EFEFEF} 0.5780 & \cellcolor[HTML]{EFEFEF} 0.1141 & \cellcolor[HTML]{EFEFEF} 0.1173 \\
\midrule

\multirow{14}{*}{\rotatebox[origin=c]{90}{\textbf{Qwen-2.5-7B}}}
& \multicolumn{6}{l}{\underline{\textit{Zero-shot Baselines}}} \\
& \quad Verbalized Confidence     & 0.1823 & 0.6723 & 0.4486 & 0.1846 & 0.2276 \\
& \quad Verbalized Confidence w/ CoT & 0.1571 & 0.7241 & 0.4866 & 0.1693 & 0.1965 \\
& \quad Self-consistency (N=10)   & 0.1916 & 0.7091 & 0.4765 & 0.2168 & 0.2314 \\
& \quad Self-consistency (N=30)   & 0.1840 & 0.7251 & 0.4861 & 0.2075 & 0.2212 \\
& \quad Logits                    & 0.1775 & 0.7382 & 0.4982 & 0.2102 & 0.2133 \\
\cmidrule(lr){2-7}
& \multicolumn{6}{l}{\underline{\textit{Calibration Baselines}}} \\
& \quad Temperature Scaling       & 0.1173 & 0.7402 & 0.4982 & 0.1529 & 0.1646 \\
& \quad Contextual Calibration    & 0.1551 & 0.7632 & 0.5159 & 0.1888 & 0.1893 \\
& \quad Batch Calibration         & 0.1796 & 0.7402 & 0.4978 & 0.2129 & 0.2162 \\
\cmidrule(lr){2-7}
& \multicolumn{6}{l}{\underline{\textit{Ours}}} \\
& \cellcolor[HTML]{EFEFEF} \quad SFT                       & \cellcolor[HTML]{EFEFEF} 0.1143 & \cellcolor[HTML]{EFEFEF} 0.8264 & \cellcolor[HTML]{EFEFEF} 0.5590 & \cellcolor[HTML]{EFEFEF} 0.1341 & \cellcolor[HTML]{EFEFEF} 0.1394 \\
& \cellcolor[HTML]{EFEFEF} \quad SFT w/ CoT                & \cellcolor[HTML]{EFEFEF} 0.1033 & \cellcolor[HTML]{EFEFEF} 0.8122 & \cellcolor[HTML]{EFEFEF} \bfseries 0.6075 & \cellcolor[HTML]{EFEFEF} 0.1095 & \cellcolor[HTML]{EFEFEF} 0.1324 \\
& \cellcolor[HTML]{EFEFEF} \quad RL (Brier)          & \cellcolor[HTML]{EFEFEF} \bfseries 0.0893 & \cellcolor[HTML]{EFEFEF} \bfseries 0.8575 & \cellcolor[HTML]{EFEFEF} 0.5804 & \cellcolor[HTML]{EFEFEF} \bfseries 0.1015 & \cellcolor[HTML]{EFEFEF} \bfseries 0.1103 \\
& \cellcolor[HTML]{EFEFEF} \quad RL (Log)            & \cellcolor[HTML]{EFEFEF} 0.1192 & \cellcolor[HTML]{EFEFEF} 0.8244 & \cellcolor[HTML]{EFEFEF} 0.5580 & \cellcolor[HTML]{EFEFEF} 0.1472 & \cellcolor[HTML]{EFEFEF} 0.1474 \\
\bottomrule
\end{tabular}
}
\vspace{-3mm}
\end{table*}

\paragraph{Metrics.} We %
assess the following three aspects 
of the autorater:
(1) its \textit{alignment} to the target preference distribution, as measured by the Mean Squared Error (MSE) between the predicted $p_\vtheta(\vx)$ and the high fidelity estimate $\hat p(\vx)$ of the true preference 
distribution $p^*(\vx)$,
(2) its \textit{performance}, in terms of agreement (Agr.) with the majority label (i.e., the mode of the target distribution) and the resulting F1 score,
and (3) its \textit{calibration}, as measured by Expected Calibration Error (ECE)~\citep{guo2017calibration} and Brier score.
A more detailed discriptions of these metrics are provided in Appendix~\ref{sec:appendix_impl_details}.

\paragraph{Distribution-Matching Tuning Improves Preference Calibration.}
As shown in Table~\ref{tab:main_results}, our distribution-matching finetuning methods consistently outperform both zero-shot and calibration baselines across all metrics. While prompting strategies like Chain-of-Thought and self-consistency improve upon simple verbalized confidence, they still result in high alignment errors. In contrast, our finetuning approach drastically reduces this error. For instance, RL with a Brier reward achieves an MSE of just $0.0764$ on Gemma-2-9B. This superior alignment translates directly into stronger performance and better calibration. Notably, on Gemma-2-9B, the RL-Brier model attains the highest agreement ($0.8706$), F1 score ($0.5895$), and lowest Brier score ($0.0946$) among all methods. The benefits are even more pronounced on Qwen-2.5-7B, where the same model achieves an agreement of $0.8575$---a $12.4\%$ improvement over Contextual Calibration, the best-performing baseline---while simultaneously achieving best calibration, as measured by both ECE and Brier score.

\begin{table}[t]
\centering
\caption{Evaluation of positional bias. We report \emph{Consistency} (higher is better) and expected \emph{Absolute Symmetry Deviation} ($\E[|\Delta_{SD}|]$, lower is better).
}
\label{tab:positional_bias}

\vspace{-2mm}

\sisetup{detect-weight, mode=text} %
\renewrobustcmd{\bfseries}{\fontseries{b}\selectfont}

\resizebox{0.72\linewidth}{!}{
\begin{tabular}{@{}l
                S[table-format=1.4] S[table-format=1.4]
                S[table-format=1.4] S[table-format=1.4]@{}}
\toprule
\multirow{2}{*}{\textbf{Method}} &
\multicolumn{2}{c}{\textbf{Gemma-2-9B}} &
\multicolumn{2}{c}{\textbf{Qwen-2.5-7B}} \\
\cmidrule(lr){2-3} \cmidrule(l){4-5}
& {\textit{Consistency} $\uparrow$} & {\textit{Abs.\ Dev.} $\downarrow$}
& {\textit{Consistency} $\uparrow$} & {\textit{Abs.\ Dev.} $\downarrow$} \\
\midrule
\multicolumn{5}{@{}l}{\underline{\textit{Zero-shot Baselines}}} \\
\quad Verbalized Confidence           & 0.7301 & 0.2242 & 0.4964 & 0.3461 \\
\quad Verbalized Confidence w/ CoT    & 0.8094 & 0.1709 & 0.6399 & 0.3120 \\
\quad Logits                           & 0.7963 & 0.1912 & 0.6529 & 0.3388 \\
\midrule
\addlinespace
\multicolumn{5}{@{}l}{\underline{\textit{Calibration Baselines}}} \\
\quad Temperature Scaling             & 0.7963 & 0.1239 & 0.6489 & 0.1953 \\
\quad Contextual Calibration          & 0.7021 & 0.3202 & 0.7422 & 0.2473 \\
\quad Batch Calibration               & 0.7973 & 0.1893 & 0.6549 & 0.2130 \\
\addlinespace
\midrule
\multicolumn{5}{@{}l}{\underline{\textit{Ours}}} \\
\rowcolor{tablelightgrey} \quad SFT               & 0.8375 & 0.1875 & 0.8284 & 0.1827 \\
\rowcolor{tablelightgrey} \quad SFT w/ CoT        & 0.7803 & 0.2291 & 0.8335 & 0.1654 \\
\rowcolor{tablelightgrey} \quad RL (Brier)  & \bfseries 0.8926 & \bfseries 0.1026 & \bfseries 0.9007 & \bfseries 0.0964 \\
\rowcolor{tablelightgrey} \quad RL (Log)    & 0.8776 & 0.1231 & 0.8726 & 0.1259 \\
\bottomrule
\end{tabular}
}
\vspace{-4mm}
\end{table}

\paragraph{RL on Binary Labels is More Annotation-Efficient than SFT on Probabilistic Labels.}
A key finding is that for a fixed annotation budget, RL is a more annotation-efficient training paradigm than SFT. As seen in Table~\ref{tab:main_results}, RL-tuned autoraters, trained on 50K prompts with a single binary label each, consistently outperform their SFT counterparts, which were trained on 5K prompts with 10 aggregated labels each. We attribute this to the benefits of data diversity: the performance boost from seeing a $10\times$ larger set of unique prompts appears to outweigh the advantage of learning from a less
\begin{wrapfigure}{r}{0.36\textwidth}
  \vspace{-0.6\baselineskip}
  \centering
  \includegraphics[width=\linewidth]{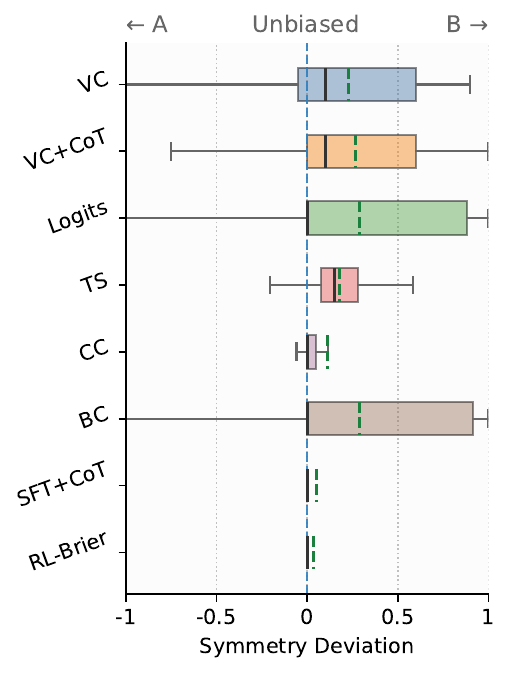}
  \vspace{-8mm}
    \caption{Positional bias by method for Qwen-2.5-7B. Each horizontal box shows the distribution of \emph{Symmetry Deviation} ($\Delta_{SD}$): $0$ is swap-symmetric, $-1$ indicates bias toward A, and $+1$ toward B. The \textbf{black solid} line marks the \emph{median}, while the \textcolor{green!50!black}{\textbf{green dashed}} line marks the \emph{mean}.}
  \label{fig:bias-box}
  \vspace{-16mm}
\end{wrapfigure}
noisy, aggregated target on a smaller dataset. Within the RL framework, the Brier reward consistently yields better results than the Log reward. This is likely because the Log reward’s heavy penalties for
tail miscalibrations
can introduce training instability, whereas the Brier reward provides a smoother optimization landscape.

\subsection{Evaluation of Positional Bias}
\label{subsec:bias_eval}
LM-based autoraters, even those based on powerful proprietary models such as GPT-4, are known to be susceptible to positional bias~\citep{zheng2023judging,wang2024large}, which causes their final verdict to be dependent on the order of the responses and thus undermines the reliability of their judgment result.
To evaluate the positional bias of the probabilistic autoraters, for each prompt $\vx$,
we perform inference twice by swapping the order of the responses to obtain two predicted probabilities: $p_{\text{orig}}$ that predicts $\Pr[B \succ A\mid \vx]$ and $p_{\text{swap}}$ that predicts $\Pr[A \succ B\mid \vx]$.

\paragraph{Metrics.}
Following \citet{zheng2023judging}, we measure \emph{consistency}, i.e., the ratio of cases where the autorater gives consistent verdicts when swapping the order of the two responses.
Additionally, we measure \emph{Symmetry Deviation} ($\Delta_{SD}$) as $\Delta_{SD} \coloneqq p_{\text{orig}} + p_{\text{swap}} - 1$.
Ideally, an unbiased autorater should produce $p_{\text{orig}}$ and $p_{\text{swap}}$ that sum to $1$, and thus the estimated $\Delta_{SD}$ would approximate $0$. A positive $\Delta_{SD}$ indicates bias toward response B, and vice versa.
We report the expected \emph{Absolute Symmetry Deviation} across the dataset as
$\E[|\Delta_{SD}|] \approx \frac{1}{N}\sum_{i=1}^N \lvert{p_{\text{orig}, i} + p_{\text{swap}, i} - 1}\rvert.$

\paragraph{Distribution-Matching Tuning Reduces Positional Bias.}
Probabilistic autoraters, like their discrete counterparts, are susceptible to positional bias, as shown in Table~\ref{tab:positional_bias}. This bias can be severe. For example, the zero-shot verbalized confidence method on Qwen-2.5-7B yields a poor consistency of just $0.4964$. While other baselines, including CoT prompting and post-hoc calibration, can mitigate this issue, significant bias remains. For instance, on Qwen-2.5-7B, no single baseline excels at both metrics, with Contextual Calibration achieving the highest consistency ($0.7422$) and Temperature Scaling achieving the lowest deviation ($0.1953$).
By contrast, our distribution-matching finetuning nearly eliminates this bias. The RL-Brier model, in particular, achieves a consistency of $0.9007$ and a near-perfect absolute symmetry deviation of $0.0964$. This dramatic improvement is visualized in Figure~\ref{fig:bias-box}, which shows that baseline methods exhibit heavily skewed deviation distributions, indicating a systematic preference for one response position (in this case, response B). Conversely, our finetuned models center the distribution tightly around zero, demonstrating robust swap-symmetry and verifying their effectiveness at debiasing autoraters.

\subsection{Out-of-Distribution Evaluation on Human-Annotated Data}
\label{subsec:ood_human}
To validate our approach against genuine human judgments, we conduct an out-of-distribution evaluation on the PandaLM test set~\citep{wang2024pandalm} using our autoraters finetuned on JudgeLM
\begin{wraptable}{r}{0.5\textwidth}
\centering
\caption{Comparison of autorater performance on the PandaLM~\citep{wang2024pandalm} test set based on human-annotated data. We report Agreement, Precision, Recall, and F1 Score. Results marked by \textsuperscript{*} are reported by \citet{zhu2025judgelm}.
}
\label{tab:pandalm_results}
\sisetup{detect-weight, mode=text} %
\renewrobustcmd{\bfseries}{\fontseries{b}\selectfont} %
\vspace{-2mm}
\resizebox{\linewidth}{!}{
\begin{tabular}{@{}l S[table-format=1.4] S[table-format=1.4] S[table-format=1.4] S[table-format=1.4]@{}}
\toprule
\textbf{Method} & {\textbf{Agreement}$\uparrow$} & {\textbf{Precision}$\uparrow$} & {\textbf{Recall}$\uparrow$} & {\textbf{F1}$\uparrow$} \\
\midrule
\multicolumn{5}{@{}l}{\underline{\textit{Zero-shot Baselines}}} \\
\quad GPT-3.5\textsuperscript{*}   & 0.6296 & 0.6195 & 0.6359 & 0.5820 \\
\quad GPT-4\textsuperscript{*}     & 0.6647 & 0.6620 & 0.6815 & 0.6180 \\
\midrule
\multicolumn{5}{@{}l}{\underline{\textit{Finetuned Baselines}}} \\
\quad PandaLM-7B\textsuperscript{*} & 0.5926 & 0.5728 & 0.5923 & 0.5456 \\
\quad JudgeLM-7B\textsuperscript{*} & 0.6507 & 0.6689 & \bfseries 0.7195 & 0.6192 \\
\midrule
\multicolumn{5}{@{}l}{\underline{\textit{Ours (Gemma-2-9B)}}} \\
\rowcolor{tablelightgrey} \quad SFT         & 0.6856 & 0.7103 & 0.5196 & 0.4998 \\
\rowcolor{tablelightgrey} \quad SFT w/ CoT  & 0.7247 & 0.6533 & 0.6166 & 0.6266 \\
\rowcolor{tablelightgrey} \quad RL (Brier)  & 0.7317 & 0.6983 & 0.6048 & 0.6220 \\
\rowcolor{tablelightgrey} \quad RL (Log)    & \bfseries 0.7357 & 0.4923 & 0.5487 & 0.5176 \\
\midrule
\multicolumn{5}{@{}l}{\underline{\textit{Ours (Qwen-2.5-7B)}}} \\
\rowcolor{tablelightgrey} \quad SFT         & 0.7027 & 0.4720 & 0.5240 & 0.4947 \\
\rowcolor{tablelightgrey} \quad SFT w/ CoT  & 0.7187 & 0.6358 & 0.6522 & \bfseries 0.6417 \\
\rowcolor{tablelightgrey} \quad RL (Brier)  & 0.7297 & \bfseries 0.8185 & 0.5617 & 0.5564 \\
\rowcolor{tablelightgrey} \quad RL (Log)    & 0.7157 & 0.8129 & 0.5454 & 0.5361 \\
\bottomrule
\end{tabular}
}
\vspace{-8mm}
\end{wraptable}
subset. This benchmark contains 1K samples, each independently annotated by three human experts. Following the standard protocol for this dataset, we treat the majority vote as the ground truth and report agreement, precision, recall, and F1 score.

\paragraph{Calibrated Autoraters Are Better Aligned with Human Preference.}
As shown in Table~\ref{tab:pandalm_results}, our models demonstrate superior alignment with human judgments compared to both powerful zero-shot models like GPT-4 and specialized, finetuned judges, including PandaLM-7B~\citep{wang2024pandalm} and JudgeLM-7B~\citep{zhu2025judgelm}. The results are particularly compelling given the data efficiency of our method. For example, our Qwen-2.5-7B model, after SFT with CoT, achieves a state-of-the-art F1 score of $0.6417$. This performance surpasses JudgeLM-7B, a model trained on the full 100K JudgeLM training set, i.e., $20\times$ more data than what's used by our SFT model. Even without CoT, our SFT model achieves an agreement of $0.7027$, outperforming all baselines, including GPT-4. These findings confirm that our distribution-matching framework is a highly data-efficient method for aligning autoraters with nuanced human preferences.

\subsection{Out-of-Distribution Evaluation on Objective Tasks}
\label{subsec:objective}
To assess performance on tasks with a single ground truth (i.e., the target preference distribution degenerates to a single point), we evaluate our models on JudgeBench~\citep{tan2025judgebench}, a 
\begin{wraptable}{r}{0.6\textwidth}
\centering
\caption{Evaluation of autoraters on JudgeBench~\citep{tan2025judgebench} across four objectively labeled tasks: Knowledge, Reasoning, Mathematics, and Coding. We report evaluation accuracy in percentage. Results marked by \textsuperscript{*} are reported by \citet{tan2025judgebench}.
}
\label{tab:judgebench}
\vspace{-2mm}
\resizebox{\linewidth}{!}{
\begin{tabular}{@{}l
                S[table-format=2.2]
                S[table-format=2.2]
                S[table-format=2.2]
                S[table-format=2.2]
                S[table-format=2.2]@{}}
\toprule
\textbf{Method} & \textbf{Knowledge} & \textbf{Reasoning} & \textbf{Math} & \textbf{Coding} & \textbf{Overall} \\
\midrule
\multicolumn{6}{@{}l}{\underline{\textit{Zero-shot Baselines}}} \\
\quad GPT-4o\textsuperscript{*}                 & 44.16 & 47.96 & 66.07 & {\bfseries 61.90} & 50.86 \\
\quad Gemini-1.5-pro\textsuperscript{*} & 45.45 & 44.90 & 53.57 & 28.57 & 44.57 \\
\midrule
\multicolumn{6}{@{}l}{\underline{\textit{Multi-Agent Baseline}}} \\
\quad ChatEval\textsuperscript{*}               & 32.47 & 31.63 & 44.64 & 30.95 & 34.00 \\
\midrule
\multicolumn{6}{@{}l}{\underline{\textit{Finetuned Baselines}}} \\
\quad PandaLM-7B\textsuperscript{*}                &  9.09 & 21.43 &  7.14 & 16.67 & 13.14 \\
\quad Prometheus2-7B\textsuperscript{*}         & 38.31 & 25.51 & 35.71 & 42.86 & 34.86 \\
\quad JudgeLM-7B\textsuperscript{*}             & 23.38 & 29.59 & 32.14 & 11.90 & 25.14 \\
\quad AutoJ-13B\textsuperscript{*}                  & 40.26 & 29.59 & 44.64 & 28.57 & 36.57 \\
\quad Skywork-Critic-8B\textsuperscript{*}  & {\bfseries 51.30} & 54.08 & {\bfseries 73.21} & 33.33 & {\bfseries 53.43} \\
\midrule
\multicolumn{6}{@{}l}{\underline{\textit{Ours (RL w/ Brier)}}} \\
\rowcolor{tablelightgrey} \quad Qwen-2.5-7B                &  39.61 & 46.94 &  60.71 & 38.10 & 44.86 \\
\rowcolor{tablelightgrey} \quad Gemma-2-9B         & 39.61 & {\bfseries 55.10} & 58.93 & 35.71 & 46.57 \\
\bottomrule
\end{tabular}
}
\vspace{-3mm}
\end{wraptable}
benchmark comprising four objective tasks (Knowledge, Reasoning, Mathematics, and Coding) with binary, verifiable labels. This benchmark allows for comparison against a diverse set of models, including state-of-the-art proprietary APIs such as GPT-4o and Gemini-1.5-pro,
multi-agent judges
such as ChatEval~\citep{chan2024chateval}, and several specialized finetuned judges, including PandaLM-7B~\citep{wang2024pandalm}, Prometheus2-7B~\citep{kim2024prometheus2}, JudgeLM-7B~\citep{zhu2025judgelm}, AutoJ-13B~\citep{li2024generative}, and Skyuwork-Critic-8B~\citep{skyworkcritic2024}.
Following the official protocol~\citep{tan2025judgebench}, we mitigate positional bias by evaluating each response pair twice, with swapped response order, and aggregating the results to obtain the final verdict.

\paragraph{Calibrated Autoraters Remain Performant on Objective Tasks.}
As shown in Table~\ref{tab:judgebench}, training our autoraters to model preference distributions does not compromise their performance on objective tasks. Our RL-Brier tuned Gemma-2-9B model, for instance, achieves the \emph{highest} accuracy of any model on the reasoning task ($55.10\%$). Its overall accuracy of $46.57\%$ surpasses strong baselines like Gemini-1.5-pro and all other finetuned judges except for Skywork-Critic-8B, which was trained on a substantially larger dataset\footnote{While the exact size of the training data is not disclosed, the Skywork-Critic-8B model is described as having been finetuned on an array of high-quality datasets, including the Skywork-Reward-Preference dataset (80K samples), the Open-Critic-GPT dataset (55K samples), and other human-annotated and synthetic data.}. Our Qwen-2.5-7B model is also highly competitive, achieving an overall accuracy of $44.86\%$. These results demonstrate that our calibration framework produces versatile probabilistic autoraters that excel at judging subjective tasks without sacrificing their effectiveness on objective, fact-based evaluations.

\section{Related Work}
\label{sec:related_work}

\paragraph{Uncertainty of Human Annotations.} There is a growing recognition that human-annotated data is not monolithic. Researchers have highlighted the importance of modeling label ambiguity and disagreement in standard classification tasks \citep{nie2020can,baan2022stop,zhou2022distributed}. Most relevant to our work, \citet{elangovan2025beyond} argue that standard correlation metrics for evaluating LLM judges are insufficient and propose new metrics that stratify data based on human label uncertainty. We take the next step by not only evaluating with respect to this uncertainty, but also proposing methods to directly train models to capture it.

\paragraph{Calibration of LLMs.} The calibration of LLMs is a well-studied problem. Early work focused on post-hoc calibration methods or prompting strategies to elicit confidence \citep{tian2023just,xiong2024can}. Other approaches use supervised fine-tuning to teach models to express uncertainty, for example by using a model's own empirical accuracy as a target label~\citep{lin2022teaching}. More recently, reinforcement learning has been used to improve calibration. \citet{xu2024sayself} use a quadratic reward function with PPO to calibrate a model for question answering. \citet{tao2024trust} combine a ranking loss with an order-preserving reward to align a model's outputs. In the context of reward modeling, \citet{leng2025taming} address the overconfidence of reward models directly by proposing PPO variants to align quality with verbalized confidence.
In concurrent work, \citet{stangel2025rewarding,damani2025beyond} propose designing rewards with proper scoring rules to improve confidence calibration on question-answering tasks.
Our work differs by focusing specifically on the autorater calibration problem and by proposing a framework grounded in modeling the true distribution of human preferences, rather than a single notion of correctness.

\paragraph{LLM-as-a-Judge.} The use of powerful large language models as automated evaluators has been explored extensively. This includes the creation of benchmarks~\citep{zheng2023judging,zeng2024evaluating,tan2025judgebench}, analyses of various biases~\citep{wang2024large,ye2024justice}, and methods for training specialized judge models~\citep{wang2024pandalm,kim2024prometheus,zhu2025judgelm,li2024generative,saha2025learning}.
Additionally, recent work has also investigated the role of model uncertainty~\citep{xie2025empirical} and non-transitivity~\citep{xu2025investigating} in LLM judge evaluations.
Our work contributes to this line by addressing a fundamental aspect of judge reliability: its calibration to the inherent subjectivity of the evaluation task.

\section{Conclusion}

In this paper, we address the core limitation of training autoraters on discrete preference labels, a practice that overlooks the subjective and distributional nature of human judgment. We introduced a general probabilistic framework to calibrate autoraters to model the full preference distribution.
Our empirical results show that finetuning with distribution-matching objective leads to autoraters that are better aligned with target preference distributions, with significant improvements in calibration and a substantial reduction in positional bias.
By shifting the goal from predicting a single verdict to modeling the spectrum of human opinion, this work contributes to building more reliable, fair, and robust AI alignment systems.

\section*{Acknowledgments}
The authors would like to thank Kai Chen from Google DeepMind for his helpful discussion and constructive feedback during the early stage of this project.

\section*{Author Contributions}
ZL and YX conceptualized the project and designed the experiments. ZL conducted the experiments and analysis. ZL and YX summarized the major findings and drafted the manuscript with inputs from XL, CH, GL, KG, BD, DS, PZ, HP, YS, and PG. DS, MK, and BAM provided intellectual input and contributed to manuscript revisions. All authors approved the final version. 

\bibliographystyle{abbrvnat}
\nobibliography*
\bibliography{references}

\appendix

\newpage

\section{Extended Related Work}

\paragraph{Disagreement in Human Judgments.}
Psychology studies have shown that even when individuals are presented with the same evidence, they can rationally arrive at different choices for what is the ``best'' because of systematic differences in their judgmental processes~\citep{mumpower1996expert},
including (i) different problem definitions, where disagreement stems from judging different environmental criteria or a confusion between scientific facts and personal value, (ii) different organizing principles, where individuals may apply different cue weights, function forms, or overall biases when integrating the same set of information, as explained by Social Judgment Theory~\citep{brehmer1976social}, and (iii) different mental models about how the evidence was generated.
In our context, this explains why human annotators can disagree about LLM outputs even under identical prompts or instructions: they face different value trade-offs (e.g., safety vs. helpfulness) or apply different thresholds for judgment, which may result in distinct yet internally coherent choices of what is preferred.

\paragraph{Probabilistic Models of Agreement.}
Probabilistic modeling of agreement dates back to the 1950s.
Classic models such as the \citet{dawid1979maximum} model aim to infer a single ``gold standard'' label from multiple, often noisy, annotators.
Studies~\citep{passonneau2014benefits,paun2018probabilistic} have shown that by modeling annotator reliability, these methods can produce high-quality data even from non-expert crowd workers and outperform simpler aggregation techniques such as majority vote.
In contrast to this approach, a recent line of work shows that, for many complex and subjective tasks, disagreement is not simply noise but a valid and reproducible signal that reflects legitimate differences in human interpretation~\citep{pavlick2019inherent,nie2020can}.
Our probabilistic autorater aligns with this view by treating the annotator heterogeneity as the prediction target to capture the full distribution of human judgments.

\paragraph{Learning from Human Feedback.}
Canonical Reinforcement Learning from Human Feedback (RLHF)~\citep{ouyang2022training} involves learning a reward model from pairwise human preferences, often by fitting a Bradley-Terry model via maximum likelihood estimation.
This approach implicitly assumes that heterogeneous feedback from different humans is merely a noisy estimate of a single ground-truth preference.
However, a growing body of work~\citep{munos2024nash,ge2024axioms,siththaranjan2024distributional} suggests that, in the context of AI alignment, preference heterogeneity often reflects legitimate differences in individual values and should be modeled directly rather than averaged away.
Our research complements this line of work as our finetuned probabilistic autoraters can be employed in such preference-based RL to better align models with the diversity of human preferences.

\section{Proofs}

\subsection{Preference Distribution and Estimation}
\label{app:pref_est}

Let $h$ index an annotator drawn from a population distribution $p(h)$. For a pairwise input $\mX=\vx$, define the annotator-specific preference probability
\[
  p_h(\vx) \;=\; \Pr[Y=1 \mid \mX=\vx, h], \qquad Y\in\{0,1\}\ \text{($1$ indicates $B \succ A$)}.
\]
The population (ground-truth) preference is the annotator-average:
\[
  p^*(\vx) \;=\; \E_{h\sim p(h)}[\,p_h(\vx)\,].
\]
A common special case here is when annotators hold \emph{stable} preferences, i.e.,
where $p_h(\vx)\in\{0,1\}$ (each annotator has a fixed judgment per $\vx$). In that case, $p_h(\vx)$ is the indicator of ``$h$ prefers B'' and $p^*(\vx)$ is the population fraction preferring B.

Given $m$ i.i.d.\ labels $\{y^{(j)}\}_{j=1}^m$ collected by sampling annotators $h_j\stackrel{\text{i.i.d.}}{\sim} p(h)$ and then $y^{(j)}\sim\mathrm{Bernoulli}(p_{h_j}(\vx))$, the Monte Carlo estimator
\[
  \hat p_m(\vx) \;=\; \frac{1}{m}\sum_{j=1}^m y^{(j)}
\]
is an unbiased estimate of $p^*(\vx)$ with variance decreasing as $1/m$.

\begin{lemma}[Unbiasedness and variance of the multi-annotator estimate]
\label{lem:mc_estimator}
For any fixed $\vx$ and i.i.d.\ sampling as above,
\[
  \E\!\left[\hat p_m(\vx)\right] \;=\; p^*(\vx),
  \qquad
  \Var\!\left[\hat p_m(\vx)\right] \;=\; \frac{p^*(\vx)\bigl(1-p^*(\vx)\bigr)}{m}.
\]
\end{lemma}

\begin{proof}
By the law of total expectation, $\E[y^{(j)}\mid\vx]=\E_h[p_h(\vx)]=p^*(\vx)$, so $\E[\hat p_m(\vx)]=p^*(\vx)$. Since the $y^{(j)}$ are i.i.d.\ Bernoulli with mean $p^*(\vx)$ (marginalizing over $h$),
$\Var(\hat p_m(\vx))=\Var(y^{(1)})/m = p^*(\vx)(1-p^*(\vx))/m$.
\end{proof}

\subsection{Proof of Proposition~\ref{prop:unified-consistency}}
\label{app:proposition_proof}

We first recall the setup. For $(\mX,Y)\sim\mathcal{D}$ with $Y\in\{0,1\}$ and
\(
Y\mid \mX=\vx \sim \mathrm{Bernoulli}(p^*(\vx)),
\)
the policy $\vpi_\vtheta(\vtau\mid \vx)$ emits a token sequence $\vtau$ intended to encode a numeric probability. A deterministic parser
\( g:\mathcal{T}\to[0,1]\cup\{\bot\} \)
returns either a number $p\in[0,1]$ or the unparsable symbol $\bot$. Let
\( s_\vtheta(\vx)=\Pr_{\vtau\sim\vpi_\vtheta(\cdot\mid \vx)}[g(\vtau)\neq\bot] \).
The piecewise rewards are:
\[
\begin{aligned}
R_{\mathrm{Brier}}(\vtau;y)
&=
\begin{cases}
1-(p-y)^2, & g(\vtau)=p\in[0,1],\\
0, & g(\vtau)=\bot,
\end{cases}
\\[4pt]
R_{\mathrm{Log}}(\vtau;y)
&=
\begin{cases}
y\log p' + (1-y)\log(1-p'), & g(\vtau)=p\in[0,1],\\
\log\epsilon, & g(\vtau)=\bot,
\end{cases}
\end{aligned}
\]

with $p'=\mathrm{clip}(p,\epsilon,1-\epsilon)$ and $\epsilon\in(0,\tfrac12)$. The population objective is
\( J_R(\vtheta)=\E_{(\mX,Y)}\,\E_{\vtau\sim\vpi_\vtheta(\cdot\mid \mX)}[R(\vtau;Y)]. \)

Fix $\vx$ and abbreviate $p^*=p^*(\vx)$. All statements below are conditional on $\mX=\vx$ and the conclusion holds for almost every $\vx$ (w.r.t.\ the marginal of $\mX$).

This proof utilizes the following observations:
(i) for any random variable $Z$ with finite variance and any constant $a$, $\E[(Z-a)^2]=(\E [Z]-a)^2+\Var(Z)$;
(ii) the function $\phi(p)=p^*\log p+(1-p^*)\log(1-p)$ is strictly concave on $p\in(\epsilon,1-\epsilon)$ with unique maximizer at $p=p^*$ (and at the boundary when $p^*\notin(\epsilon,1-\epsilon)$).

\paragraph{Brier.}
Let $P$ denote the random numeric report $g(\vtau)$ conditional on $g(\vtau)\neq\bot$ (so $P$ is defined with probability $s_\vtheta(\vx)$). Then
\begin{align*}
\E_{\vtau,Y}\!\left[R_{\mathrm{Brier}}(\vtau;Y)\mid \vx\right]
&= s_\vtheta(\vx)\,\E_{\vtau}\!\left[\,\E_{Y\mid \vx}\bigl[1-(P-Y)^2\bigr]\right] + \bigl(1-s_\vtheta(\vx)\bigr)\cdot 0\\
&= s_\vtheta(\vx)\,\E_{\vtau}\!\left[1-\,\E_{Y\mid \vx}\bigl[(Y-P)^2\bigr]\right] \\
&= s_\vtheta(\vx)\,\E_{\vtau}\!\left[1-\,\left(\E_{Y\mid \vx}[Y]-P\right)^2-\Var_{Y\mid \vx}(Y)\right] && \text{applying (i)} \\
&= s_\vtheta(\vx)\,\E_{\vtau}\!\left[1-\,\left(p^*-P\right)^2-p^*(1-p^*)\right] \\ %
&= s_\vtheta(\vx)\,\Bigl( 1 - \E_{\vtau}\!\left[(P-p^*)^2\right] - p^*(1-p^*) \Bigr) \\
&= s_\vtheta(\vx)\,\Bigl( 1 - \bigl[(\E_{\vtau}[P] - p^*)^2 + \Var_{\vtau}(P)\bigr] - p^*(1-p^*) \Bigr). && \text{applying (i)}
\end{align*}
For fixed $s_\vtheta(\vx)$ this is maximized by setting $\Var(P)=0$ (deterministic numeric output) and $\E[P]=p^*$ (truthful reporting). Moreover, since
$1-p^*(1-p^*) > 0$, the expectation is positive when $\Var(P)=0$ and $\E[P]=p^*$. Increasing $s_\vtheta(\vx)$ strictly increases the expectation; hence an optimizer satisfies $s_\vtheta(\vx)=1$. Therefore, at any global maximizer,
$P$ is almost surely constant and equals $p^*$, i.e., $p_\vtheta(\vx)=p^*$.

\paragraph{Log with clipping.}
Write $P'$ for the clipped numeric report when parsable. Then
\[
\begin{aligned}
\E_{\vtau,Y}\!\left[R_{\mathrm{Log}}(\vtau;Y)\mid \vx\right]
&= s_\vtheta(\vx)\,\E_{\vtau}\!\left[\,\E_{Y|\vx}\left[Y\log P' + (1-Y)\log(1-P')\right]\,\right] + (1-s_\vtheta(\vx))\,\log\epsilon \\
&= s_\vtheta(\vx)\,\E_{\vtau}\!\left[\,p^*\log P' + (1-p^*)\log(1-P')\,\right] + (1-s_\vtheta(\vx))\,\log\epsilon \\
&= s_\vtheta(\vx)\,\E_{\vtau}[\phi(P')] + (1-s_\vtheta(\vx))\log\epsilon .
\end{aligned}
\]
By strict concavity of $\phi$ and Jensen's inequality,
\(
\E[\phi(P')] \le \phi(\E [P'])
\)
with equality iff $P'$ is almost surely constant (deterministic numeric output). The maximizer over $P'\in[\epsilon,1-\epsilon]$ is uniquely
\(
P'\equiv \mathrm{clip}(p^*,\epsilon,1-\epsilon)
\)
(truthful reporting).
Finally, at this maximizer
\(
\phi\!\bigl(\mathrm{clip}(p^*,\epsilon,1-\epsilon)\bigr) > \log\epsilon,
\)
so allocating any mass to unparsable outputs (which yields $\log\epsilon$) strictly reduces the expectation; hence $s_\vtheta(\vx)=1$.
Therefore, the optimal report is deterministic and equals the clipped truth $\mathrm{clip}(p^*,\epsilon,1-\epsilon)$.

\smallskip
Combining the two cases completes the proof. \hfill$\square$

\section{Implementation Details}
\label{sec:appendix_impl_details}

\subsection{Evaluation Details}

We measure the probabilistic autorater's performance by comparing its judgment to human annotation or discretized probabilistic labels annotated by Gemini-2.5-Flash.
Following prior work~\citep{wang2024pandalm,zhu2025judgelm}, we formulate the pairwise judgment task as a three-class classification problem with labels $A \succ B$, Tie, and $B \succ A$. Let $TP_i$, $FP_i$, and $FN_i$ denote the true positives, false positives, and false negatives for class $i$, respectively. We report macro-averaged metrics by computing each score per class and then averaging over all $C$ classes:
\begin{align*}
\text{Precision}_{\text{macro}} &= \frac{1}{C} \sum_{i=1}^C \frac{TP_i}{TP_i + FP_i}, \\
\text{Recall}_{\text{macro}} &= \frac{1}{C} \sum_{i=1}^C \frac{TP_i}{TP_i + FN_i}, \\
\text{F1-score}_{\text{macro}} &= \frac{1}{C} \sum_{i=1}^C 
   \frac{2 \cdot TP_i}{2 \cdot TP_i + FP_i + FN_i}, \\
\text{Agreement} &= \frac{\sum_{i=1}^{C} TP_i}{\sum_{i=1}^{C} (TP_i + FN_i)}.
\end{align*}

By convention, calibration is measured with respect to a set of discrete (binary) labels.
To evaluate model calibration, we first binarize the preference distribution to obtain ground truth preference labels and then measure the Expected Calibration Error (ECE) and Brier Score.
Test samples with ground truth label being ``Tie'' are skipped for calibration evaluation.

The ECE is calculated by dividing the confidence into $K$ equal-sized bins (each of size $\tfrac{1}{K}$), and then calculating the accuracy and average confidence within each bin:
\begin{align*}
\mathrm{ECE}
  &= \sum_{k=1}^{K} \frac{|B_k|}{N}\,\bigl| \mathrm{Acc}(B_k) - \mathrm{Conf}(B_k) \bigr|, \\
\mathrm{Acc}(B_k)
  &= \frac{1}{|B_k|} \sum_{i \in B_k} \mathbf{1}\!\left(\hat{y}_i = y_i\right), \qquad
\mathrm{Conf}(B_k)
  = \frac{1}{|B_k|} \sum_{i \in B_k} \hat{p}_i,
\end{align*}
where $B_k$ is the number of samples whose prediction confidence falls into the interval $(\tfrac{k-1}{K}, \tfrac{k}{K}]$, $\hat{y}_i$ and $y_i$ are the predicted and true preference labels, and $\hat{p}_i$ is the predicted probability.
By default, we set $K=10$ in our experiments.

The Brier score is calculated as $\tfrac{1}{N}\sum_{i=1}^N (y_i - \hat{p}_i)^2$.

\subsection{Finetuning Details}

In Setting 2, we optimize the following GRPO objective~\citep{shao2024deepseekmath}
\[
\mathcal{J}(\vtheta)
= \mathbb{E}_{\vx \sim p(\vx),\vtau \sim \vpi_{\vtheta_{\mathrm{old}}}}
\!\left[
\min\!\left(
\frac{\vpi_\vtheta(\vtau|\vx)}{\vpi_{\vtheta_{\mathrm{old}}}(\vtau|\vx)}\, A_t,\;
\operatorname{clip}\!\left(\frac{\vpi_\vtheta(\vtau|\vx)}{\vpi_{\vtheta_{\mathrm{old}}}(\vtau|\vx)},\, 1-\varepsilon,\, 1+\varepsilon\right) A_t
\right)
- \beta\, D_{\mathrm{KL}}\!\left(\vpi_\vtheta \,\|\, \vpi_{\mathrm{ref}}\right)
\right],
\]
with the unbiased KL estimator~\citep{schulman2020approximating}
\[
D_{\mathrm{KL}}\!\left(\vpi_\vtheta \,\|\, \vpi_{\mathrm{ref}}\right)
\approx
\frac{\vpi_{\mathrm{ref}}(\vtau|\vx)}{\vpi_\vtheta(\vtau|\vx)}
- \log\!\frac{\vpi_{\mathrm{ref}}(\vtau|\vx)}{\vpi_\vtheta(\vtau|\vx)}
- 1,
\]

where \(\vpi_\vtheta\) is the policy model being optimized,
\(\vpi_{\vtheta_{\mathrm{old}}}\) is the previous policy model,
\(\vpi_{\mathrm{ref}}\) is the reference policy,
\(A_t\) is the advantage estimate,
\(\varepsilon\) is the clipping hyperparameter,
and \(\beta\) is the KL penalty coefficient.
Detailed parameter settings for our experiments are presented in Table~\ref{tab:finetune_setting}.

\begin{table}[h]
\centering
\caption{Detailed finetuning settings.}
\label{tab:finetune_setting}
\small
\begin{tabular}{@{}l p{0.45\linewidth}@{}}
\toprule

\multicolumn{2}{@{}l}{\textbf{Setting 1 — Supervised Finetuning (SFT)}} \\
\midrule
\underline{General} & \\
\quad max total sequence length & 2048 \\
\quad precision & bf16 \\
\underline{Optimization} & \\
\quad optimizer & AdamW~\citep{loshchilov2019decoupled} \\
\quad optimizer hyperparameters & $\beta_1,\beta_2,\epsilon = 0.9,\,0.98,\,10^{-5}$ \\
\quad weight decay & 0.1 \\
\quad batch size & 32 \\
\quad training epochs & 3 \\
\quad learning rate & $1\times 10^{-6}$ \\

\midrule
\multicolumn{2}{@{}l}{\textbf{Setting 2 — Reinforcement Learning (RL)}} \\
\midrule
\underline{General} & \\
\quad max total sequence length & 2048 \\
\quad precision & bf16 \\
\underline{GRPO Setting} & \\
\quad hyperparameters & $\beta=0.01,\varepsilon=0.2$ \\
\quad number of prompts per step & 32 (Gemma) / 64 (Qwen) \\
\quad number of generations per prompt & 32 \\
\underline{Optimization} & \\
\quad optimizer & AdamW~\citep{loshchilov2019decoupled} \\
\quad optimizer hyperparameters & $\beta_1,\beta_2,\epsilon = 0.9,\,0.999,\,10^{-8}$ \\
\quad weight decay & 0.001 \\
\quad batch size & 512 \\
\quad training epochs & 1 \\
\quad learning rate & $3\times 10^{-7}$ \\
\quad learning rate warm-up & linear \\ 
\quad warm-up ratio / steps & 0.01 \,/\, 50 \\
\bottomrule
\end{tabular}
\end{table}

\section{Additional Results}

\subsection{Win Rate Prediction}
We use a subset of the LMSys Chatbot Arena Conversation dataset~\citep{zheng2023judging,kahng2025comparator} to evaluate the autorater's ability to predict the LM's win rate.
This subset contains a total of 900 prompts for comparing responses from Gemma 1.0 and Gemma 1.1.
For this evaluation, we use Qwen-2.5-7B as the base model for the autorater to avoid potential self-enhancement bias (i.e., favoring responses generated by LMs from the same family).
Specifically, given two LMs $\vpi_A$ and $\vpi_B$, the win rate (of $\vpi_B$) is defined as $$\Pr[\vpi_B\succ \vpi_A]\coloneqq \mathbb{E}_{\vx \sim p(\vx)}\mathbb{E}_{\vtau_A\sim \vpi_A, \vtau_B\sim \vpi_B} \Pr[\vtau_B \succ \vtau_A|\vx].$$
From the results shown in Table~\ref{tab:win_rate}, we observe that the predicted win rates from the finetuned autoraters are more aligned with the true win rate voted by human judges.

\begin{table}[h]
\centering
\caption{Evaluation on the Chatbot Arena Conversations dataset for comparing Gemma 1.0 to Gemma 1.1. We compare the autorater's predicted win rate for Gemma 1.0 against the true win rate from human annotators.
}
\label{tab:win_rate}
\sisetup{detect-weight, mode=text}
\renewrobustcmd{\bfseries}{\fontseries{b}\selectfont}
\small
\begin{tabular}{@{}l S[table-format=1.4] S[table-format=1.4]@{}}
\toprule
& {\textbf{Win Rate}} & {\textbf{Absolute Error}} \\
\textbf{Method} & {\textbf{of Gemma 1.0}} & {\textbf{to Human}} \\
\midrule
Human & 0.4344 & \multicolumn{1}{c}{---} \\
\midrule
Qwen-2.5-7B Verbal & 0.7397 & 0.3053 \\
Qwen-2.5-7B Verbal w/ CoT & 0.5951 & 0.1607 \\
Qwen-2.5-7B SFT & 0.3146 & 0.1198 \\
Qwen-2.5-7B SFT w/ CoT & 0.3082 & 0.1262 \\
Qwen-2.5-7B RL (Brier) & 0.3640 & 0.0704 \\
Qwen-2.5-7B RL (Log) & \bfseries 0.3662 & \bfseries 0.0682 \\
\bottomrule
\end{tabular}
\end{table}

\subsection{RL with Probabilistic Labels}
While our main experiments utilized sparse binary labels for RL, our reward function is also compatible with dense probabilistic labels. To explore how this data format affects performance, we finetuned Qwen-2.5-7B using the Brier score reward on the same 5K prompts (each with 10 annotations) used for SFT.

As shown in Figure~\ref{fig:data_format}, we compared this model against several baselines: the zero-shot model with direct verbalized confidence, the SFT model trained on probabilistic labels, and our primary model trained with RL on binary labels. The results indicate that RL with probabilistic labels improves both performance (agreement) and calibration (ECE) over the zero-shot baseline. However, it underperformed compared to the autoraters trained with a larger set of binary labels, highlighting the crucial role of data diversity in achieving optimal results.

\begin{figure}[ht]
    \centering
    \includegraphics[width=0.6\linewidth]{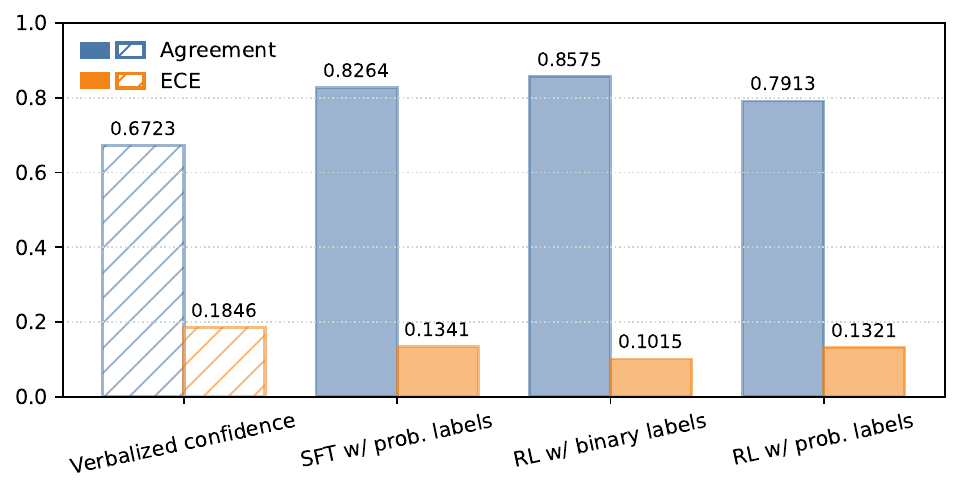}
    \caption{Result of RL finetuning with probabilistic labels.}
    \label{fig:data_format}
\end{figure}

\section{Discussions}

We provide a discussion on the extensions and limitations of our approach and results.

First, our work initiates the study on more reliable autoraters by predicting the full preference distribution. While our empirical results focused on the pairwise evaluation setting, the general framework and analysis also apply to the pointwise evaluation setting, where the preference takes nominal or ordinal values (e.g., Likert scale). Specifically, a direction extension involves employing multi-class strictly proper scoring rules, such as the multi-class Brier score or cross-entropy, as the reward function for RL. We leave a thorough empirical study in this space as future work.

Second, our finetuning objective aims to capture the distribution of human preferences, which stems from the human-level uncertainty on the subject. As such, the uncertainty communicated by the finetuned autorater in the form of verbalized confidence is rather aleatoric than epistemic, whereas the latter may require predicting a second-order distribution (i.e., a distribution over the predicted probability $p_\vtheta(\vx)$). However, such an approach may significantly complicate training; instead, we demonstrate empirically that simply providing a point estimation of the probability is sufficient to improve the autorater's alignment to the preference distribution while maintaining good out-of-distribution generalizability.

\section{Prompts}
\label{sec:appendix_prompt}

\begin{tcolorbox}[
    breakable,
    colback=gray!5!white,
    colframe=gray!60!black,
    fonttitle=\bfseries,
    title=Prompt Set for Persona-Based Preference Annotation,
    arc=2mm,
    boxrule=0.8pt,
    left=5mm,
    right=5mm,
]

\noindent\textbf{Base Prompt Structure:}
\begin{tcolorbox}[
    colback=black!5!white,
    boxrule=0.2pt,
    sharp corners,
    fontupper=\scriptsize\ttfamily,
]

\{Persona-specific instruction here\}

Your output must be a single letter, either 'A' or 'B', 
enclosed in \textless winner\textgreater\ and \textless/winner\textgreater\ tags.

[Question]
\{question\}

[The Start of Response A]
\{responseA\}
[The End of Response A]

[The Start of Response B]
\{responseB\}
[The End of Response B]

Considering these factors, which response is better?

\end{tcolorbox}

\hrulefill\medskip %

\noindent\textbf{Persona-Specific Instructions:}
\begin{description}
    \item[\textbf{1. Factual \& Technical Specialist}] {Your task is to evaluate two responses based on their helpfulness, relevance, accuracy, and level of detail. As a technical specialist, you would additionally consider \textbf{factual correctness and precision} in your judgment.}

    \item[\textbf{2. Safety \& Privacy Advocate}] Your task is to evaluate two responses based on their helpfulness, relevance, accuracy, and level of detail. As a safety advocate, you would additionally consider \textbf{safety and ethical responsibility}. An unsafe response is a poor response.
    
    \item[\textbf{3. Empathetic Communicator}] Judge the overall quality of the responses based on their helpfulness, accuracy, and detail. As an empathetic communicator, you would also pay close attention to the \textbf{tone and respectfulness} of the language.
    
    \item[\textbf{4. Concision Champion}] Assess which response is better, considering its helpfulness, accuracy, and relevance. As a champion of concision, you would additionally prefer the response that is more \textbf{direct and to-the-point}, penalizing unnecessary verbosity.
    
    \item[\textbf{5. Nuance \& Depth Seeker}] Evaluate the helpfulness, relevance, and accuracy of both responses. As an academic researcher, you would also give importance to the \textbf{depth and nuance} of the explanation.
    
    \item[\textbf{6. Creative \& Arts Critic}] Assess which response is more helpful, relevant, and accurate. As a creative critic, you would also weigh the \textbf{originality, style, and engagement} of the writing.
    
    \item[\textbf{7. Everyday User} (4x Weight)] From the perspective of a regular person, not an expert, evaluate which response is better overall. Your decision should be based on which one is most \textbf{helpful, relevant to the question, accurate, and provides a good level of detail}. The best response is the one that is most satisfying and clear.
    
\end{description}
\end{tcolorbox}

\begin{tcolorbox}[
    colback=black!5!white,   %
    colframe=black!75!white, %
    title=\textbf{Prompts for Eliciting Probabilistic Prediction},
    fonttitle=\bfseries,
    breakable,               %
    pad at break=2mm,
    fontupper=\ttfamily\small %
]

\subsection*{1. Verbalized Confidence}
\begin{lstlisting}
Your task is to evaluate two responses, A and B, based on the provided question. Compare the responses and determine the probability that **Response B is better than Response A**.

Your output must be a single float between 0.0 and 1.0, enclosed in `<prob_rb_better>` and `</prob_rb_better>` tags. Do not output any other text, reasoning, or explanations.
- A probability of 0.0 means Response A is clearly better.
- A probability of 1.0 means Response B is clearly better.
- A probability of 0.5 indicates they are of similar quality.

Example Output Format:
<prob_rb_better>0.85</prob_rb_better>

[Question]
{question}

[The Start of Response A]
{responseA}
[The End of Response A]

[The Start of Response B]
{responseB}
[The End of Response B]

Probability Output:
\end{lstlisting}

\hrule %

\subsection*{2. Verbalized Confidence w/ CoT}
\begin{lstlisting}
Your task is to evaluate two responses, A and B, based on the provided question. Compare the responses on helpfulness, accuracy, and clarity.

Follow these steps for your output:
1. First, provide a step-by-step reasoning of your evaluation inside `<think>` and `</think>` tags. Explain which response is better and why.
2. Immediately after the closing `</think>` tag, provide a single float between 0.0 and 1.0 representing the probability that **Response B is better than Response A**. This float must be enclosed in `<prob_rb_better>` and `</prob_rb_better>` tags. A probability of 0.0 means Response A is clearly better, 1.0 means Response B is clearly better, and 0.5 indicates they are of similar quality.

Example Output Format:
<think>Your detailed reasoning comparing A and B goes here.</think><prob_rb_better>0.85</prob_rb_better>

[Question]
{question}

[The Start of Response A]
{responseA}
[The End of Response A]

[The Start of Response B]
{responseB}
[The End of Response B]

Your Evaluation:
\end{lstlisting}

\hrule %

\subsection*{3. Self-Consistency}
\begin{lstlisting}
Your task is to evaluate two responses, A and B, based on the provided question. You must decide which response is superior.
First, provide a step-by-step thinking process that justifies your decision. This reasoning should be enclosed within <think></think> tags. 
After your reasoning, provide your final decision, which must be a single character ('A' or 'B'), enclosed within <answer></answer> tags.

Example Output Format:
<think>Response A directly answers the user's question about X, while Response B is more general. Response A also provides a concrete example which is helpful. Therefore, A is the better response.</think><answer>A</answer>

[Question]
{question}

[The Start of Response A]
{responseA}
[The End of Response A]

[The Start of Response B]
{responseB}
[The End of Response B]

Final Evaluation:
\end{lstlisting}

\hrule %

\subsection*{4. Logit-Based Confidence}
\begin{lstlisting}
Your task is to evaluate two responses, A and B, based on the provided question. 
You must decide which response is superior. Your output must be a single character: 'A' or 'B'. 
Do not output any other text, reasoning, or explanations.

Example Output Format:
B

[Question]
{question}

[The Start of Response A]
{responseA}
[The End of Response A]

[The Start of Response B]
{responseB}
[The End of Response B]

Final Decision (A or B):
\end{lstlisting}

\end{tcolorbox}

\section{Response Examples}

\subsection{JudgeLM Dataset}

\begin{autoraterbox}[\textcolor{blue}{Gemma-2-9B} (JudgeLM \emph{val} set - QID 827)]{blue}

\majoritem{Question}
Compose a haiku with the provided subject.
Spring

\medskip
\majoritem{Model A (vicuna-13b-1.1) Response}
\begin{rawblock}
Fresh new start, green leaves grow
Raindrops pitter-patter, flowers bloom.
\end{rawblock}

\majoritem{Model B (mix-instruct-output) Response}
\begin{rawblock}
Amidst blooming flowers
Gentle breeze moves budding trees
Spring awakens life.
\end{rawblock}

\medskip
\majoritem{Ground Truth}
$\displaystyle P(\text{B is better}) = 1.00$

\medskip
\majoritem{Autoraters Outputs}

\textbf{Verbalized Confidence w/ CoT (Prediction: 0.60)}
\begin{rawblock}
<think>Both responses follow the 5-7-5 syllable structure of a haiku. 

Response B is slightly better because it paints a more vivid picture with its imagery. "Amidst blooming flowers" and "Gentle breeze moves budding trees" create a stronger sense of the season's atmosphere compared to Response A's more general descriptions.  Response B also uses more evocative language like "awakens life," which adds a deeper layer of meaning. </think><prob_rb_better>0.6</prob_rb_better> 

\end{rawblock}
\textbf{SFT w/ CoT (Prediction: 1.00)}
\begin{rawblock}
<think>1.  **Analyze Response A:** Response A attempts to follow the 5-7-5 syllable structure but fails. "Fresh new start, green leaves grow" has 7 syllables, "Raindrops pitter-patter, flowers bloom" has 9 syllables. It also includes a line break within the second line, which is not a haiku characteristic.
2.  **Analyze Response B:** Response B correctly follows the 5-7-5 syllable structure for all three lines. It also uses evocative imagery and language appropriate for the subject "Spring."
3.  **Compare and Justify Score:** Response B is a perfect haiku, adhering to the strict structural and thematic requirements. Response A completely fails on the syllable count and line structure, making it not a haiku at all. Therefore, Response B is unequivocally better than Response A. A confidence score of 1.00 is warranted as there is no scenario where Response A would be preferred for this task.</think><prob_rb_better>1.0</prob_rb_better>
\end{rawblock}
\textbf{RL (Brier) (Prediction: 0.99)}
\begin{rawblock}
<think>Response B is a better haiku because it follows the traditional structure of a haiku (5-7-5 syllables per line) while Response A does not.  Response B also captures the essence of spring with its imagery of blooming flowers, gentle breezes, and awakening life. Response A's lines are too long and don't have the same poetic flow. </think><prob_rb_better>0.99</prob_rb_better> 

\end{rawblock}
\textbf{RL (Log) (Prediction: 1.00)}
\begin{rawblock}
<think>Response B is better because it adheres to the 5-7-5 syllable structure of a haiku, while Response A does not. Response B also uses more evocative language and imagery to capture the essence of spring. Response A is more of a descriptive paragraph than a haiku.</think><prob_rb_better>1.0</prob_rb_better> 



\end{rawblock}

\end{autoraterbox}

\begin{autoraterbox}[\textcolor{blue}{Gemma-2-9B} (JudgeLM \emph{val} set - QID 3960)]{blue}

\majoritem{Question}
Tell me what are aspects of polytopes from given text
In geometry, a simplicial polytope is a polytope whose facets are all simplices. For example, a simplicial polyhedron in three dimensions contains only triangular faces and corresponds via Steinitz's theorem to a maximal planar graph.

They are topologically dual to simple polytopes. Polytopes which are both simple and simplicial are either simplices or two-dimensional polygons.

\medskip
\majoritem{Model A (mix-instruct-output) Response}
\begin{rawblock}
1. A simplicial polytope is a polytope in geometry where each facet is a simplice.
2. As an illustration, a simplicial polyhedron in three dimensions with only triangular faces corresponds to a maximum planar network according to Steinitz's theorem.
3. They are simple polytopes topological dual.
4. Two-dimensional polygons or simplices are polytopes that are both simple and simplicial.
\end{rawblock}

\majoritem{Model B (alpaca-native) Response}
\begin{rawblock}
The aspects of polytopes from given text are: a polytope whose facets are all simplices and topologically dual to simple polytopes. Polytopes which are both simple and simplicial are either simplices or two-dimensional polygons.
\end{rawblock}

\medskip
\majoritem{Ground Truth}
$\displaystyle P(\text{B is better}) = 0.90$

\medskip
\majoritem{Autoraters Outputs}

\textbf{Verbalized Confidence w/ CoT (Prediction: 0.70)}
\begin{rawblock}
<think>Both responses attempt to summarize the aspects of polytopes from the provided text. 

Response A breaks down the information into numbered points, making it slightly easier to read and understand. However, it contains some inaccuracies:

* Point 3 is incomplete and grammatically incorrect. It should read "They are topologically dual to simple polytopes."
* Point 2 uses the phrase "maximum planar network" which is likely a typo and should be "maximal planar graph."

Response B, while less structured, accurately summarizes the key aspects of polytopes mentioned in the text. It is concise and avoids the errors present in Response A. </think><prob_rb_better>0.7</prob_rb_better> 




\end{rawblock}
\textbf{SFT w/ CoT (Prediction: 0.40)}
\begin{rawblock}
<think>1.  **Analyze the Question:** The user wants to extract "aspects of polytopes" from the provided text. This implies identifying key definitions, properties, or relationships related to polytopes.

2.  **Evaluate Response A:**
    *   **Helpfulness & Relevance:** Response A attempts to list aspects. It correctly identifies "simplicial polytope" and "triangular faces." However, it includes "maximum planar network" instead of "maximal planar graph," which is a minor factual error. It also lists "topological dual" and "two-dimensional polygons or simplices."
    *   **Coherence & Quality:** Response A is presented as a numbered list, which is a good format for extracting distinct points. The language is clear and directly quotes or paraphrases the text.

3.  **Evaluate Response B:**
    *   **Helpfulness & Relevance:** Response B attempts to list aspects. It correctly identifies "a polytope whose facets are all simplices" and "topologically dual to simple polytopes." It also includes "Polytopes which are both simple and simplicial are either simplices or two-dimensional polygons."
    *   **Coherence & Quality:** Response B is a single, concise sentence. While it successfully condenses the information, it lacks the explicit separation of distinct aspects that a numbered list provides. It also uses the phrase "The aspects of polytopes from given text are:" which is a bit redundant given the question.

4.  **Compare and Justify Score (0.40):**
    *   Response A is slightly better because it presents the information in a more structured and easily digestible format (a numbered list). This makes it clearer that each point is a distinct "aspect."
    *   Response B is more concise but less organized. Its single sentence format makes it harder to parse as a list of distinct aspects.
    *   The minor factual error in Response A ("maximum planar network" instead of "maximal planar graph") is not significant enough to outweigh the superior presentation and clarity of Response A.
    *   Therefore, Response A is marginally better, leading to a confidence score of 0.40 (meaning B is better with 40%
\end{rawblock}
\textbf{RL (Brier) (Prediction: 0.98)}
\begin{rawblock}
<think>Response B is better because it directly answers the question in a concise and clear manner. It identifies the key aspects of polytopes mentioned in the text without unnecessary elaboration or rephrasing. Response A, while accurate, provides a more verbose and less focused summary of the text. </think><prob_rb_better>0.98</prob_rb_better> 

\end{rawblock}
\textbf{RL (Log) (Prediction: 0.65)}
\begin{rawblock}
<think>Response B is better because it directly answers the question by extracting and summarizing the key aspects of polytopes mentioned in the text. Response A, while accurate, simply rephrases the information from the text without providing a clear and concise summary. Response B is more focused and highlights the essential concepts related to polytopes. </think><prob_rb_better>0.65</prob_rb_better> 



\end{rawblock}

\end{autoraterbox}

\begin{autoraterbox}[\textcolor{forestgreen}{Qwen-2.5-7B} (JudgeLM \emph{val} set - QID 376)]{green}

\majoritem{Question}
What is an anemone?

\medskip
\majoritem{Model A (alpaca-native) Response}
\begin{rawblock}
An anemone is a type of sea creature that is typically found in marine habitats. It is typically described as a large, beautiful, and often colorful animal that is found on coral reefs. Anemones are typically classified as either solitary or colonial, depending on the species. Solitary anemones are typically smaller and are found in shallow waters, while colonial anemones are larger and can be found in deeper waters. The anemone uses its stinging cells to capture prey, and is also known to benefit from the symbiotic relationship it has with certain types of algae.
\end{rawblock}

\majoritem{Model B (vicuna-13b-1.1) Response}
\begin{rawblock}
An anemone is a type of marine animal that belongs to the phylum Cnidaria and the class Anthozoa. It is a multicellular, radially symmetrical animal that has a simple, tubular body and a single, central oral-aboral opening surrounded by tentacles. Anemones are found in coral reefs, on the seafloor, and in shallow waters. They are predators, feeding on a variety of small sea creatures. An anemone's color can range from brown to red, orange, or purple.
\end{rawblock}

\medskip
\majoritem{Ground Truth}
$\displaystyle P(\text{B is better}) = 0.70$

\medskip
\majoritem{Autoraters Outputs}

\textbf{Verbalized Confidence w/ CoT (Prediction: 0.95)}
\begin{rawblock}
<think>
Response A provides a basic description of an anemone, mentioning its habitat and classification into solitary or colonial types. However, it lacks detail about the anemone's physical characteristics and behavior. Response B, on the other hand, gives a more detailed and scientifically accurate description, mentioning the phylum and class of anemones, their physical structure, and feeding habits. It also provides a range of colors, which is more specific and helpful for understanding the anemone's appearance. Therefore, Response B is more accurate and detailed.
</think>
<prob_rb_better>0.95</prob_rb_better>
\end{rawblock}
\textbf{SFT w/ CoT (Prediction: 0.90)}
\begin{rawblock}
<think>1.  **Analyze Response A:** Response A describes an anemone as a "large, beautiful, and often colorful animal." This is a significant factual error. Anemones are typically small, often less than a few inches in diameter, and not generally described as "large." This misrepresentation of size is a major flaw.
2.  **Analyze Response B:** Response B provides a more accurate and detailed description. It correctly identifies anemones as belonging to the phylum Cnidaria and class Anthozoa, and describes their physical characteristics (multicellular, radially symmetrical, tubular body, tentacles, oral-aboral opening). It also correctly states their habitat and feeding habits.
3.  **Compare Helpfulness & Relevance:** Response B is far more helpful and relevant due to its accuracy. Response A's fundamental error regarding size makes it misleading and unhelpful.
4.  **Compare Coherence & Quality:** Response B is well-structured, clear, and accurate. Response A contains a critical factual error that undermines its quality.
5.  **Justify Score (0.90):** Response B is significantly better because it provides accurate and comprehensive information, while Response A contains a major factual error regarding the size of an anemone. The error in Response A makes it largely unhelpful and potentially misleading. Therefore, Response B is overwhelmingly superior.</think><prob_rb_better>0.9</prob_rb_better>
\end{rawblock}
\textbf{RL (Brier) (Prediction: 0.99)}
\begin{rawblock}
<think>
Response A is less accurate and less clear. It incorrectly describes anemones as "beautiful" and "large," which are not typical descriptors for anemones. It also incorrectly states that solitary anemones are typically smaller and found in shallow waters, while colonial anemones are larger and found in deeper waters, which is not a general rule. Response B, on the other hand, provides a more accurate and clear description, mentioning that anemones belong to the phylum Cnidaria and class Anthozoa, and correctly describes their physical characteristics and habitat. It also correctly states that anemones are predators and that their color can vary.

Given these differences, Response B is clearly better.
</think><prob_rb_better>0.99</prob_rb_better>
\end{rawblock}
\textbf{RL (Log) (Prediction: 1.00)}
\begin{rawblock}
<think>
Response A is less accurate and less clear. It incorrectly describes anemones as "beautiful" and "large," which is not a defining characteristic. It also incorrectly states that solitary anemones are typically smaller and found in shallow waters, while colonial anemones are larger and found in deeper waters, which is not a general rule. Response B, on the other hand, provides a more precise and accurate description, including the scientific classification and the correct habitat and feeding habits. It also correctly mentions the range of colors anemones can have.

</think><prob_rb_better>1.0</prob_rb_better>
\end{rawblock}

\end{autoraterbox}

\begin{autoraterbox}[\textcolor{forestgreen}{Qwen-2.5-7B} (JudgeLM \emph{val} set - QID 3311)]{green}

\majoritem{Question}
How can I get rid of a rash on my neck?

\medskip
\majoritem{Model A (moss-moon-003-sft) Response}
\begin{rawblock}
There are several ways to get rid of a rash on your neck. Here are some options:

1. Try over-the-counter creams: Over-the-counter creams like hydrocortisone cream or chapstick can help soothe and moisturize the rash. Be sure to read the labels carefully and follow the instructions.

2. Use aloe vera: Aloe vera juice can help soothe and moisturize the skin, making it less itchy. Mix one tablespoon of aloe vera gel with one tablespoon of water and apply to the rash.

3. Use tea tree oil: Tea tree oil has anti-inflammatory properties that can help reduce inflammation and itching. Apply a few drops of tea tree oil to the rash and leave it on for 30 minutes before rinsing it off.

4. Get a rash shot: In some cases, your doctor may recommend a rash shot to help reduce the itching and inflammation. This is a prescription medication that you should consult your doctor about before taking.

5. Try home remedies: There are several home remedies that can help soothe a rash on your neck. For example, you can try applying cucumber slices to the rash to help
\end{rawblock}

\majoritem{Model B (oasst-sft-4-pythia-12b-epoch-3.5) Response}
\begin{rawblock}
You should see an allergist to get tested for allergies.
\end{rawblock}

\medskip
\majoritem{Ground Truth}
$\displaystyle P(\text{B is better}) = 0.50$

\medskip
\majoritem{Autoraters Outputs}

\textbf{Verbalized Confidence w/ CoT (Prediction: 0.15)}
\begin{rawblock}
<think>
Response A provides a detailed list of methods to alleviate a rash on the neck, including over-the-counter creams, aloe vera, tea tree oil, rash shots, and home remedies. It offers a comprehensive approach and is helpful and clear for someone dealing with a rash. Response B, on the other hand, suggests seeing an allergist, which is a more specific and targeted approach but does not provide immediate relief or practical steps for managing the rash. While Response B is accurate in suggesting a professional consultation, it lacks the practical, immediate advice that Response A provides.

Given the context of the question, which asks for ways to get rid of a rash, Response A is more helpful and comprehensive. However, Response B is still accurate and could be a valuable follow-up step after trying the methods in Response A.
</think>
<prob_rb_better>0.15</prob_rb_better>
\end{rawblock}
\textbf{SFT w/ CoT (Prediction: 0.90)}
\begin{rawblock}
<think>1.  **Analyze Response A:** Response A provides a list of potential solutions, including OTC creams, aloe vera, tea tree oil, and even a "rash shot." While some of these suggestions (like tea tree oil) are generally safe for topical use, others (like "rash shot") are vague and potentially misleading without proper medical consultation. The response also cuts off abruptly.
2.  **Analyze Response B:** Response B suggests seeing an allergist. This is a highly relevant and potentially crucial step for diagnosing and treating a rash, especially if it's allergy-related. It prioritizes professional medical advice.
3.  **Compare Helpfulness & Relevance:** Response B is significantly more helpful and relevant because it points the user towards a professional diagnosis, which is often the first and most important step in treating a rash. Response A's suggestions, while not inherently harmful, are less direct and potentially less effective without a proper diagnosis.
4.  **Compare Coherence & Quality:** Response A is a bit disjointed due to the abrupt ending and the inclusion of potentially confusing or vague advice (e.g., "rash shot"). Response B is concise and directly addresses the underlying cause, making it more coherent and actionable.
5.  **Justify Score (0.90):** Response B is overwhelmingly better because it prioritizes professional medical advice, which is the most responsible and effective way to address a rash. Response A's suggestions, while not actively harmful, are less direct and potentially less effective without a proper diagnosis. The score of 0.90 reflects that B is significantly superior due to its focus on professional guidance.</think><prob_rb_better>0.9</prob_rb_better>
\end{rawblock}
\textbf{RL (Brier) (Prediction: 0.75)}
\begin{rawblock}
<think>
Response A provides a detailed list of methods to treat a rash on the neck, including over-the-counter creams, aloe vera, tea tree oil, and home remedies. However, it is quite lengthy and includes some less conventional methods like "rash shots" which are not commonly used for rashes and could be misleading. The response is clear but could be more concise.

Response B, on the other hand, suggests seeing an allergist, which is a more direct and professional approach. This is particularly useful if the rash is due to an allergic reaction. However, it doesn't provide any specific treatment options and is quite brief.

Given the question, Response A offers more comprehensive information, but Response B provides a more targeted and professional approach. Since the question asks for a specific recommendation, Response B is more likely to be helpful for the user.

Considering the above, Response B is better, but not by a wide margin.
</think><prob_rb_better>0.75</prob_rb_better>
\end{rawblock}
\textbf{RL (Log) (Prediction: 0.00)}
\begin{rawblock}
<think>
Response A provides a detailed list of methods to treat a rash on the neck, including over-the-counter creams, aloe vera, tea tree oil, and home remedies. It is clear and offers multiple options, making it helpful for the user. Response B, on the other hand, suggests seeing an allergist, which is a single action and does not provide specific treatment options. While it might be a valid step in some cases, it does not directly address the user's question about how to get rid of the rash.

Given the comprehensive nature and directness of Response A, it is more helpful and clearer for the user's needs.
</think>
<prob_rb_better>0.0</prob_rb_better>
\end{rawblock}

\end{autoraterbox}

\subsection{PandaLM Dataset}

\begin{autoraterbox}[\textcolor{blue}{Gemma-2-9B} (PandaLM \emph{test} set - QID 412)]{blue}

\majoritem{Question}
Instruction: Rewrite the given text and correct grammar, spelling, and punctuation errors.

Input: If you'd told me year ago that today I would finish a marathon, I would of laughed. Your support had a huge affect on me!

\medskip
\majoritem{Model A (bloom-7b) Response}
\begin{rawblock}
If you told me a year ago that today I would finish a marathon, I would have laughed. Your support had a huge effect on me!
\end{rawblock}

\majoritem{Model B (cerebras-gpt-6.7B) Response}
\begin{rawblock}
If you had told me a year ago that today I would finish a marathon, I would have laughed. Your support had a huge impact on me.
\end{rawblock}

\medskip
\majoritem{Ground Truth}
Preference: B better

\medskip
\majoritem{Autoraters Outputs}

\textbf{Verbalized Confidence w/ CoT (Prediction: 0.60)}
\begin{rawblock}
<think>Response B is slightly better because it correctly uses the past perfect tense ("had told") which is more grammatically accurate in this context. Response A uses the simple past tense ("told") which is less precise. Additionally, Response B uses the word "impact" which is a more precise synonym for "affect".</think><prob_rb_better>0.6</prob_rb_better> 



\end{rawblock}
\textbf{SFT w/ CoT (Prediction: 0.60)}
\begin{rawblock}
<think>1.  **Identify Errors in Input:** The input contains two main errors: "would of" should be "would have," and "affect" should be "effect" or "impact."
2.  **Evaluate Response A:** Response A correctly changes "would of" to "would have" and "affect" to "effect."
3.  **Evaluate Response B:** Response B also correctly changes "would of" to "would have" and "affect" to "impact."
4.  **Compare Responses:** Both responses correctly identify and fix the grammatical errors. However, Response B uses "impact" instead of "effect," which is a more common and widely accepted synonym for "affect" in this context (meaning influence). While "effect" is also grammatically correct, "impact" is arguably a slightly more precise and natural fit for the meaning of "huge affect."
5.  **Justify Score:** Response B is marginally better due to its choice of "impact" over "effect," making the corrected sentence slightly more precise and natural. The difference is subtle, but it gives B a slight edge in quality. Therefore, a confidence score of 0.60 is appropriate, indicating B is slightly better.</think><prob_rb_better>0.6</prob_rb_better>
\end{rawblock}
\textbf{RL (Brier) (Prediction: 0.98)}
\begin{rawblock}
<think>Response B is better because it correctly uses the past perfect tense ("had told") and the word "impact" instead of "affect," which is used incorrectly in the original text. Response A is grammatically correct but uses the incorrect past perfect tense ("told") and the word "effect" which is not the best choice in this context. </think><prob_rb_better>0.98</prob_rb_better> 



\end{rawblock}
\textbf{RL (Log) (Prediction: 1.00)}
\begin{rawblock}
<think>Response B is better because it makes more accurate grammatical corrections. 

*  Response A correctly changes "would of" to "would have" but keeps "affect" which should be "effect". 
* Response B correctly changes "would of" to "would have" and "affect" to "impact". 

Response B is also slightly more natural in its phrasing. </think><prob_rb_better>0.999</prob_rb_better> 



\end{rawblock}

\end{autoraterbox}

\begin{autoraterbox}[\textcolor{forestgreen}{Qwen-2.5-7B} (PandaLM \emph{test} set - QID 573)]{green}

\majoritem{Question}
Instruction: Predict how many stars the author will give to the restaurant from a Yelp review.

Input: The reviews were great, but honestly i felt it was just ok. Seemed like a typical tourist spot were the food was mediocre. The service was ok, not great, but not the worst.
The bread was very good. I ordered the muscle fries. They had a sweet Vidalia onion flavor. No need for fries bc they were soggy wo much flavor. My boyfriend ordered a pesto chicken pasta and it was again mediocre.

\medskip
\majoritem{Model A (cerebras-gpt-6.7B) Response}
\begin{rawblock}
5 stars
\end{rawblock}

\majoritem{Model B (opt-7b) Response}
\begin{rawblock}
The author will likely give the restaurant 3 stars.
\end{rawblock}

\medskip
\majoritem{Ground Truth}
Preference: B better

\medskip
\majoritem{Autoraters Outputs}

\textbf{Verbalized Confidence w/ CoT (Prediction: 0.95)}
\begin{rawblock}
<think>
Response A simply states "5 stars" without any reasoning or context, which is not helpful or accurate given the mixed feedback in the review. Response B, on the other hand, provides a more reasoned prediction based on the content of the review. The review mentions that the experience was "just ok" and that the food and service were "mediocre," which suggests a below-average experience. Therefore, a 3-star rating seems more appropriate than 5 stars.

Response B is clearer and more aligned with the sentiment expressed in the review.
</think>
<prob_rb_better>0.95</prob_rb_better>
\end{rawblock}
\textbf{SFT w/ CoT (Prediction: 1.00)}
\begin{rawblock}
<think>1.  **Analyze the Input:** The review is mixed. It starts with "The reviews were great, but honestly i felt it was just ok." This indicates a neutral to slightly negative overall impression. It then lists specific criticisms (mediocre food, mediocre service, soggy fries) and positive points (good bread, sweet Vidalia onion flavor). The language is generally unenthusiastic.
2.  **Evaluate Response A:** Response A predicts "5 stars." This is highly unlikely given the content of the review. The reviewer explicitly states it was "just ok" and lists multiple mediocre aspects. A 5-star rating would imply strong satisfaction, which is not evident.
3.  **Evaluate Response B:** Response B predicts "3 stars." This is a much more plausible rating for the given review. A 3-star rating typically indicates a neutral to slightly negative experience, which aligns well with the mixed and somewhat critical tone of the review.
4.  **Compare Responses:** Response B's prediction is significantly more accurate and reasonable than Response A's. Response A is completely off, while Response B is a good, educated guess based on the review's content.
5.  **Justify Score:** A confidence score of 1.00 is appropriate because Response B is definitively better than Response A. Response A's prediction is entirely incorrect and unhelpful, whereas Response B provides a plausible and well-reasoned prediction.</think><prob_rb_better>1.0</prob_rb_better>
\end{rawblock}
\textbf{RL (Brier) (Prediction: 0.99)}
\begin{rawblock}
<think>
Response A is overly optimistic and does not reflect the mixed sentiment in the review. The review mentions "great reviews" but also states the experience was "just ok" and "mediocre." The bread was good, but the fries and pasta were described as "soggy" and "mediocre." Given the overall mixed and somewhat negative tone, a 5-star rating is not appropriate.

Response B, on the other hand, suggests a 3-star rating, which seems more aligned with the mixed sentiment expressed in the review. A 3-star rating typically indicates an average experience, which fits the description of the restaurant being "just ok" and having "mediocre" food and service.

Therefore, Response B is better than Response A.
</think><prob_rb_better>0.99</prob_rb_better>
\end{rawblock}
\textbf{RL (Log) (Prediction: 1.00)}
\begin{rawblock}
<think>
Response A is overly optimistic and does not reflect the overall sentiment of the review. The review mentions that the food was "mediocre" and the service was "ok, not great." There is no indication that the author would give 5 stars. Response B, on the other hand, is more aligned with the sentiment expressed in the review. The author found the experience to be "just ok" and the food "mediocre," which suggests a middle rating. Therefore, Response B is more accurate and helpful.
</think>
<prob_rb_better>1.0</prob_rb_better>
\end{rawblock}

\end{autoraterbox}

\end{document}